\documentclass[twoside,11pt]{article}
\usepackage{journal2e_v1}

\usepackage{url}
\usepackage{epsf}
\usepackage{epsfig}
\usepackage{amsmath}
\usepackage{subfigure}
\usepackage{amssymb}
\usepackage{algorithm}
\usepackage{multirow}
\usepackage{multicol}

\usepackage{epstopdf}
\usepackage{makecell}

\usepackage{graphicx} 

\usepackage{algorithm}
\usepackage{algorithmic}

\def\A{{\bf A}}

\def\b{{\bf b}}
\def\C{{\bf C}}

\def\D{{\bf D}}

\def\I{{\bf I}}

\def\X{{\bf X}}

\def\s{{\bf s}}

\def\x{{\bf x}}
\def\y{{\bf y}}
\def\z{{\bf z}}

\def\u{{\bf u}}

\def\v{{\bf v}}

\def\w{{\bf w}}
\def\0{{\bf 0}}
\def\1{{\bf 1}}

\def\AM{{\mathcal A}}

\def\MM{{\mathcal M}}

\def\RB{{\mathbb R}}
\def\NB{{\mathbb N}}

\def\epsi{\mbox{\boldmath$\epsilon$\unboldmath}}

\def\Tha{\mbox{\boldmath$\Theta$\unboldmath}}

\def\Si{\mbox{\boldmath$\Sigma$\unboldmath}}

\def\vps{\mbox{\boldmath$\varepsilon$\unboldmath}}

\def\argmin{\mathop{\rm argmin}}
\def\liml{\mathop{\lim}\limits}
\def\liminf{\mathop{\rm lim\;inf}}

\def\dist{\mathrm{dist}}
\def\dom{\mathrm{dom}}

\def\sgn{\mathrm{sgn}}

\def\diag{\mathrm{diag}}

\title{Nonconvex Penalization in Sparse Estimation: An Approach Based on the Bernstein Function}

\author{\name Zhihua Zhang  \\
\addr  Department of Computer Science and Engineering \\
Shanghai Jiao Tong University \\
800 Dong Chuan Road, Shanghai, China 200240 \\
\texttt{zhzhang@gmail.com}
}

\date{\today}

\begin{document}

\maketitle

\begin{abstract}%
In this paper we study nonconvex penalization using Bernstein functions whose first-order derivatives are completely monotone.
The Bernstein function can induce a class of nonconvex penalty functions
for high-dimensional sparse estimation problems.
We derive a thresholding function based on the Bernstein penalty and discuss some important mathematical properties in sparsity modeling. We show that a coordinate descent algorithm is especially appropriate for regression problems penalized by  the Bernstein function.
We also consider the application of the Bernstein penalty in classification problems
and devise a proximal alternating linearized minimization method.
Based on theory of the Kurdyka-\L{}ojasiewicz inequality, we
conduct convergence analysis of these alternating iteration procedures. 
We particularly exemplify a family of Bernstein nonconvex penalties based on a generalized Gamma measure and conduct empirical analysis for this family.
\end{abstract}

\section{Introduction}

Variable  selection  plays a fundamental role in statistical modeling for high-dimensional
data sets, especially when the underlying model has a sparse representation.
The approach based on penalty theory has been widely used for variable selection. 
A principled approach is due to the lasso of \citet{TibshiraniLASSO:1996,TibshiraniLASSO:2011}, which employs the $\ell_1$-norm penalty and performs variable selection via the soft thresholding operator.  The lasso enjoys attractive statistical properties~\citep{KnightFu:2000,ZhaoJMLR:06}.
However, \citet{Fan01} pointed out
that the lasso  method produces
biased estimates for the large coefficients. \citet{ZouJasa:2006} argued that the lasso might not
be an oracle procedure under certain scenarios.

\citet{Fan01} proposed three criteria for evaluating a good penalty function. That is, the resulting estimator should hold
\emph{sparsity}, \emph{continuity} and \emph{unbiasedness}. Moreover, \citet{Fan01} showed that a nonconvex penalty generally
admits the oracle properties. \citet{ZhangZhang2012} presented a general theoretical analysis for nonconvex regularization.
Meanwhile, there exist many nonconvex penalties, including the $\ell_q$ ($q\in (0,1)$) penalty,  the
smoothly clipped absolute deviation (SCAD) \citep{Fan01},
the minimax concave plus penalty (MCP) \citep{Zhang2010mcp}, the kinetic energy plus  penalty (KEP)~\citep{ZhangTR:2013},  the capped-$\ell_1$ function~\citep{TZhangJMLR:10,GongICML:2013},  the nonconvex exponential penalty (EXP)~\citep{BradleyICML:1998,GaoAAAI:2011},
the LOG penalty~\citep{MazumderSparsenet:11,ArmaganDunsonLee},   the smooth integration of counting and absolute deviation (SICA)~\citep{LvFan:2009},  the hard thresholding penalty~\citep{ZhengFanLv:2014}, etc.
These penalty functions have been demonstrated to have attractive properties theoretically and practically.

On the one hand, nonconvex penalty functions typically yield the tighter  approximation to the $\ell_0$-norm and hold some good theoretical properties.
On the other hand, they would make computational challenges due to  nondifferentiability and nonconvexity that they have.
Recently,  \citet{MazumderSparsenet:11}
developed  a SparseNet algorithm  based on  coordinate descent.
Specifically,  the authors studied the coordinate descent algorithm for the MCP function and conducted convergence analysis
\citep[also see][]{BrehenyAAS:2010}.
Moreover,
\citet{MazumderSparsenet:11} proposed some desirable properties for thresholding operators based on nonconvex
penalty functions. For example, the thresholding operator is expected to have a nesting property; that is,
it should be a strict nesting
w.r.t.\ a sparsity parameter (see Section~\ref{sec:threshold}).
However,  the authors  claimed that not all nonconvex penalty functions are suitable for use with  coordinate descent.

In this paper we study  Bernstein functions whose first-order derivatives are completely monotone~\citep{SSVBook:2010,FellerBook:1971}.
The Bernstein function has the L\'{e}vy-Khintchine representation~\citep{SSVBook:2010}. 
Since there is a one-to-one correspondence between Bernstein penalty functions and Laplace exponents of subordinators,
the Bernstein function has  been used to develop a nonparametric Bayesian approach to sparse estimation problems~\citep{ZhangBA:2014}.

We introduce Bernstein functions into sparse estimation, giving rise to a unified approach to  nonconvex penalization.
We particularly exemplify a family of Bernstein nonconvex penalties based on a generalized Gamma measure~\citep{Aalen:1992,Brix:1999}.
The special cases include the KEP,
nonconvex LOG and EXP as well as a
penalty function that we call  linear-fractional (LFR) function. Moreover, we find that the MCP function is
a capped  version of KEP.
More specifically, our work offers the following major contributions.
\begin{itemize}
\item Applying  the notion of regular variation~\citep{FellerBook:1971}, we establish the connection of the Bernstein function
with the $\ell_{q}$-norm ($0\leq q <1$) and the $\ell_1$-norm. Using the notion of  limiting-subdifferentials~\citep{RockafellarWets:1998}, we show that the Bernstein function enjoys the  Kurdyka-\L{}ojasiewicz property~\citep{Lojasiewicz,Kurdyka,BolteDaniilidisLewis}.
\item We prove that the Bernstein penalty function admits the oracle properties and can result in an unbiased and continuous sparse estimator.
We derive a thresholding function based on the Bernstein penalty. We show that  this thresholding operator
to some extent has  the nesting property pointed out by \citet{MazumderSparsenet:11}.
\item We present a  coordinate descent algorithm and a proximal
alternating linearized algorithm
for solving  the  regression  and  classification problems, respectively. Based on theory of the Kurdyka-\L{}ojasiewicz inequality,  we prove that these alternating iteration  procedures have global convergence properties.
Specifically, we show that the algorithms can find a strict local minimizer under certain regularity conditions. 
\end{itemize}

The remainder of this paper is organized as follows. Section~\ref{sec:prel} reviews some preliminaries that will be used.
Section~\ref{sec:lapexp} exploits Bernstein functions in the construction of nonconvex penalties.
In Section~\ref{sec:sest} we investigate sparse estimation problems based on the Bernstein function. We
devise a coordinate descent algorithm for finding the sparse regression solution and a proximal alternating linearized
minimization algorithm for solving the classification problem. In Section~\ref{sec:convergence} we show that these algorithms
enjoy a global convergence property by using the novel Kurdyka-\L{}ojasiewicz  inequality.
In Section~\ref{sec:experiment} we conduct our empirical evaluations.
Finally, we conclude our work in
Section~\ref{sec:conclusion}. All proofs are given in the appendix.
In the appendix we also present some important properties of Bernstein functions and asymptotic consistencies of  our concerned sparse estimation model.

\section{Preliminaries}
\label{sec:prel}

Suppose we are given a set of training
data $\{(\x_i, y_i): i=1,\ldots, n\}$, where
the $\x_i \in \RB^{p}$ are the input vectors and the $y_i$ are the corresponding
outputs.
We now consider the following supervised  learning model:
\[
\y = \X \b + \vps,
\]
where $\y=(y_1, \ldots, y_n)^T$ is the $n{\times}1$ output (or response) vector,
$\X=[\x_1, \ldots, \x_n]^T$ is the $n {\times} p$
input (or design) matrix, and $\vps$ is an error vector. In regression problems
 $\vps \sim N(\vps|\0, \sigma \I_n)$, and  in classification problems it is defined as  multivariate Bernoulli.
We aim at finding a sparse estimate  of regression vector $\b=( b_1, \ldots, b_p)^T$ under the regularization framework.

The classical regularization approach is based on
a penalty function of $\b$. That is,
\[
\min_{\b}  \;  \Big\{F(\b) \triangleq L(\b; \y, \X)   + \lambda_n P(\b) \Big\},
\]
where $L(\cdot)$ is the loss function, $P(\cdot)$ is
the regularization term penalizing model complexity, and $\lambda_n$ ($>0$) is the tuning  parameter of balancing the relative significance
of the loss function and the penalty. Specifically, $L(\b;  \X, \y) \triangleq \sum_{i=1}^n \ell(\b; \x_i, y_i)\triangleq \sum_{i=1}^n \frac{1}{2}(y_i- \b^T\x_i)^2 = \frac{1}{2} \|\y {-} \X \b\|_2^2$ in the regression problem, and  $L(\b;  \X, \y)\triangleq\sum_{i=1}^n \ell(\b; \x_i, y_i) \triangleq \sum_{i=1}^n \log(1+ \exp(-y_i \b^T \x_i))$ in the classification
problem where $y_i \in \{-1, 1\}$.

A widely used setting for  penalty is $P(\b)= \sum_{j=1}^p P(b_j)$,
which implies that the penalty function consists of $p$ separable subpenalties.
In order to find a sparse solution of $\b$, one imposes the $\ell_0$-norm penalty $\|\b\|_0$ to  $\b$ (i.e., the number of nonzero elements of $\b$).
However, the resulting optimization problem is usually NP-hard. Alternatively,
the $\ell_1$-norm   $\|\b\|_1=  \sum_{j=1}^p |b_j|$ is an effective convex  penalty.
Additionally, some nonconvex alternatives, such as the  LOG-penalty,
SCAD,  MCP and KEP, have been  studied. Meanwhile, iteratively reweighted $\ell_q$  ($q=1$ or $2$) minimization and coordination descent methods  were developed for finding
sparse solutions.

We take the MCP function as an example. Specifically, $P(b)=  M(\alpha |b|)$  where $\alpha$ is a positive constant and $M$ is defined as
\begin{equation} \label{eqn:mcp}
 M(|b|) = \left\{\begin{array}{ll} \frac{1}{2} & \mbox{ if } |b| \geq 1, \\
|b| - \frac{ b^2}{2} & \mbox{ if } |b| < 1.  \end{array} \right.
\end{equation}
Clearly, $M'(|b|)$  exists w.r.t.\ $|b|$. But the second derivative  $M''(|b|)$  does not exist at $|b|=1$.

For a nonconvex penalty function $P(b)$, it is interesting to exploit its maximum concavity, which is defined as
\[
\zeta(P) = \sup_{s, t \in \RB, s<t} \; - \frac{P'(t) - P'(s)}{t-s}.
\]
When $P$ is twice differentiable on $[0, \infty)$, $\zeta(P)= \sup_{s \in (0, \infty)} - P''(s)$ \citep{LvFan:2009}.   
That is, it is the maximum curvature of the curve $P$. 

We now recall the notion of subdifferentials, which can be used to define the subdifferential of
a nonconvex  function.

\begin{definition}[Subdifferentials] \citep{RockafellarWets:1998} Consider a proper and lower semi-continuous function $f: \RB^{p} \to(-\infty, +\infty]$ and a point ${\x} \in \dom(f)$.
\begin{enumerate}
\item[\emph{(i)}] The Fr\'{e}chet subdifferential of $f$ at ${\x}$, denoted $\hat{\partial} f({\x})$, is the set of all vectors $\u \in \RB^p$
which satisfy
\[
\liminf_{\begin{footnotesize}\begin{array}{c} \y \neq \x \\ \y \to \x \end{array} \end{footnotesize}}
\frac{f(\y)- f(\x)- \u^T(\y-\x) }{\|\y-\x \|} \geq 0.
\]
\item[\emph{(ii)}] The limiting-subdifferential of $f$ at ${\x}$, denoted ${\partial} f({\x})$, is defined as
\[
{\partial} f({\x}) \equiv \Big\{\u \in \RB^p: \exists \x_k \to \x, f(\x_k) \to f(\x) \; \mbox{ and } \; \u_k \in \hat{\partial} f({\u_k})
\to \u \; \mbox{ as } \; k\to \infty  \Big\}.
\]
\end{enumerate}
\end{definition}
Notice that when $\x \notin \dom f$,  $\hat{\partial} f({\x}) = \emptyset$ is the default setting. It is well established that
$\hat{\partial} f({\x}) \subseteq {\partial} f({\x})$ for each $\x \in \RB^p$, and both  $\hat{\partial} f({\x})$ and ${\partial} f({\x})$
are closed~\citep{RockafellarWets:1998}. Moreover, if $\bar{\x}$ is a critical point of $f$, then $\0 \in \partial f(\bar{\x})$.

Let us see an example in which  $f(x)=|x|^q$ with $q \in (0, 1]$. We have $\hat{\partial} f({0})= {\partial} f({0})=(-\infty, \infty)$ when $0<q<1$, and $\hat{\partial} f({0})= {\partial} f({0})=[-1, 1]$ when $q=1$. If defining $f(b) =M(|b|)$ (see Eqn.~(\ref{eqn:mcp})),
we obtain that  $\hat{\partial} f({0})= {\partial} f({0})=[-1, 1]$.

Next we briefly review the Kurdyka-\L{}ojasiewicz property, which will play an important role in our global convergence analysis.
Given $\eta \in (0, \infty]$, we let $\Pi_{\eta}$ denote the class of  continuous concave functions $\pi:[0,\eta)\to\RB_{+}$ which satisfy the following conditions:
\begin{enumerate}
\item[{(a)}] $\pi(0)=0$,
\item[{(b)}] $\pi$ is $C^{1}$ on $(0, \eta)$,
\item[{(c)}] $\pi^{\prime}(u)>0$  for all $u \in(0,\eta)$.
\end{enumerate}
Clearly, function $\pi(u)= u^{1-\gamma}$ for $\gamma \in (0,1]$ belongs to $\Pi_{\eta}$ with $\eta=+\infty$.

\begin{definition}[Kurdyka-\L{}ojasiewicz property] \label{def:klpro}
Let the function $f: \RB^{p} \to(-\infty, +\infty]$ be proper and lower semi-continuous. Then $f$ is said
to have the Kurdyka-\L{}ojasiewicz property at $\bar {\x} \in \dom \, \partial f$ if there exist $\eta \in (0, +\infty]$,
a neighborhood ${\cal U}$ of $\bar {\x}$, and a  function $\pi \in \Pi_{\eta}$
such that for all $\x \in {\cal U} \cap [f(\bar {\x} )< f< f({\bar \x})+\eta]$, the following  Kurdyka-\L{}ojasiewicz inequality holds
\[ \pi^{\prime}(f(\x)-f({\bar \x})) \dist(0, \partial f(\x))\geq 1
\]
under the  notational conventions: $0^0 =1$ and
$\infty/\infty=0/0=0$. If $f$  has the Kurdyka-\L{}ojasiewicz property at every point in $\dom \, \partial f$, then $f$ is said to have the Kurdyka-\L{}ojasiewicz property. Here $\dom (\partial f) =\{\x: \partial f(\x)\neq \emptyset\}$, and  $\dist(\v, \AM)=\inf\left\{\|\v-\u\|, \u \in\AM\right\}$.
\end{definition}

The Kurdyka-\L{}ojasiewicz property was proposed by \citet{Lojasiewicz}, who proved that for a real analytic function  $f$
and $\pi(u)=u^{1-\gamma}$
with $\gamma\in [\frac{1}{2}, 1)$
$\frac{|f(\x)- f(\bar {\x} )|^{\gamma}}{\dist(\0, \partial f(\x))}$ is bounded around any critical point $\bar{\x}$.
\citet{Kurdyka} extended this property to a class of functions with the $o$-minimal structure. \citet{BolteDaniilidisLewis}
extended to nonsmooth subanalytic functions. Recently,  the Kurdyka-\L{}ojasiewicz property
has been used to establish convergence analysis  of proximal alternating minimization for nonconvex problems~\citep{AttouchBotle,BolteMathProgram,XuYin}.

We again consider the MCP function in Eqn.(\ref{eqn:mcp}), and define $f(b) =M(|b|)$.
The graph of $f$ is the closure of the set
\begin{align*}
&  \Big\{ (y, b): y= \frac{1}{2}, \; b< - 1 \Big\} \cup \Big \{ (y, b): y= - \frac{b^2}{2} - b, \; -b<1, \; b<0 \Big \} \\
& \cup \Big \{(y, b): y= - \frac{b^2}{2} + b, \; -b <0, \; b<1 \Big \}  \cup \Big \{ (y, b): y= \frac{1}{2}, \; -b< -1 \Big \}.
\end{align*}
Thus, the graph is semialgebraic~\citep{BolteDaniilidisLewis}. This implies that  the MCP function satisfies  the Kurdyka-\L{}ojasiewicz property.
Analogously, we also obtain that the SCAD function  satisfies  the Kurdyka-\L{}ojasiewicz property.

The Huber loss function is a classical tool in robust regression. It is 
\begin{equation} \label{eqn:huber}
L_{\delta}(z) = \left\{\begin{array} {ll}  \frac{1}{2} z^2 & \mbox{ if }  |z| \leq \delta, \\
\delta |z| - \frac{1} {2} \delta^2 & \mbox{ otherwise}.   \end{array}  \right.
\end{equation}
Obviously, the   Huber loss function enjoyed the  the Kurdyka-\L{}ojasiewicz property.

\section{Bernstein Penalty Functions}
\label{sec:lapexp}

Let  ${\cal S} \subset [0, \infty)$ and $f \in C^{\infty}({\cal S})$ with $f\geq 0$.
We say $f$ to be completely monotone if $(-1)^k f^{(k)} \geq 0$ for all $k \in \NB$ and to be a Bernstein function
if $(-1)^k f^{(k)} \leq 0$ for all $k \in \NB$. It is well known that
$f$ is a Bernstein function if and only if the mapping $s \mapsto \exp(- t f(s))$ is completely monotone for all $t\geq 0$.
Additionally, $f$ is a Bernstein function if and only if it has the representation
\[
f(s) = a + \beta s + \int_{0}^{\infty} {\big[1 - \exp(-  s u) \big]  \nu(d u)} \; \mbox{ for all } s> 0,
\] 
where $a \geq 0$ and $\beta \geq 0$, and  $\nu$ is the L\'{e}vy measure satisfying additional requirements $\nu(-\infty, 0)=0$ and
$\int_{0}^{\infty} { \min(u, 1) \nu(d u) } < \infty$. Moreover, this representation  is unique.  The representation is famous as the L\'{e}vy-Khintchine formula.

Since $\liml_{s\to 0} f(s) =a$ and $\liml_{s\to \infty} \frac{f(s)}{s}=\beta$~\citep{SSVBook:2010},
we will assume that $\liml_{s \to 0}f(s)=0$ and $\liml_{s\to \infty} \frac{f(s)}{s} =0$ to make $a=0$ and $\beta=0$.
Notice that  $s^{q}$, for $q \in (0, 1)$, is a Bernstein function of $s$ on $(0, \infty)$ satisfying the above assumptions.
However, $f(s)=s$ is  Bernstein but does not satisfy  the condition $\liml_{s\to \infty} \frac{f(s)}{s} =0$. Indeed,
$f(s)=s$  is an extreme example because  $\beta=1$ and $\nu(d u) =\delta_{0}(u) d u$ (the Dirac Delta measure at the origin) in its L\'{e}vy-Khintchine formula.
In fact, the condition $\liml_{s\to \infty} \frac{f(s)}{s} =0$ aims at excluding  this Bernstein function for our concern in this paper.

\subsection{Properties}

We  now define the penalty function $P(b)$ as $\Phi(|b|)$,
where the penalty term $\Phi(s)$ is a Bernstein function  of $s$ on $[0, \infty)$ such that $\Phi(0)=0$ and $\liml_{s \to \infty} \frac{\Phi(s)}{s} =0$.
Clearly,
$\Phi(s)$
is nonnegative, nondecreasing and concave on $[0, \infty)$, because $\Phi(s) \geq 0$, $\Phi'(s)\geq 0$ and $\Phi{''}(s)\leq 0$.
As a function of $b \in \RB$, $\Phi(|b|)$ is of course continuous.
Moreover, we have the following theorem.

\begin{theorem} \label{thm:lapexp00}
Let $\Phi
(s)$ be a nonzero Bernstein function of $s$ on $[0, \infty)$. Assume $\Phi(0)=0$ and $\liml_{s\to \infty} \frac{\Phi(s)}{s} =0$.
Then
\begin{enumerate}
\item[\emph{(a)}] $\Phi(|b|)$ is a nonnegative and nonconvex function of $b$ on $(-\infty, \infty)$, and an increasing function
of $|b|$ on  $[0, \infty)$.
\item[\emph{(b)}] $\Phi(|b|)$ is continuous  w.r.t.\ $b$ but nondifferentiable at the origin.
\item[\emph{(c)}] Define $P(b) \triangleq \Phi(|b|)$. Then $\hat{\partial} P(0) = {\partial} P(0) =[-1, 1]$, and $\partial P(b) = \Phi'(|b|) \partial |b|$. Moreover,  $P(b)$ satisfies the Kurdyka-\L{}ojasiewicz property.
\end{enumerate}
\end{theorem}

Recall that under the conditions in Theorem~\ref{thm:lapexp00}, $a$ and $\beta$ in the L\'{e}vy-Khintchine formula vanish.
Theorem~\ref{thm:lapexp00} (b) says that $\Phi'(|b|)$  is singular at the origin. Thus,   $\Phi(|b|)$  can define a
class of sparsity-inducing nonconvex penalty functions.   Theorem~\ref{thm:lapexp00} (a) shows  that $\Phi(|b|)$ satisfies Condition~1 given in \citet{LvFan:2009}.  
Theorem~\ref{thm:lapexp00} (c) says that the Bernstein function has the same subdifferential with the function $|b|$
at the origin.
Moreover, the Bernstein function has the Kurdyka-\L{}ojasiewicz property. 
As mentioned earlier, however, the subdifferential of the function $|b|^{q}$ for $0<q<1$ at $b=0$ is $(-\infty, \infty)$.

We can clearly  see the connection of the bridge penalty $|b|^{q}$ with the $\ell_0$-norm and the $\ell_1$-norm,
as $q$ goes from $0$ to 1. However, the sparse estimator resulted from the bridge penalty is not continuous.
This would make numerical computations and model predictions unstable~\citep{Fan01}.
In this paper we consider another class of Bernstein  nonconvex penalty functions.

In particular, to explore the relationship of the Bernstein penalty with the $\ell_0$-norm and the $\ell_1$-norm,
we further assume that
$\Phi'(0)=\liml_{s \to 0} \Phi'(s)<\infty$.
Since $\Phi(s)$ is a nonzero Bernstein function of $s$, we can conclude that $\Phi'(0)>0$.
If it were not true, we would have $\Phi'(s)=0$ due to $\Phi'(s) \leq \Phi'(0)$.
This implies that $\Phi(s)=0$ for any $s \in (0, \infty)$
because $\Phi(0)=0$. This conflicts with that   $\Phi(s)$ is  nonzero.
Similarly, we can also deduce $\Phi{''}(0) < 0$. Based on this fact, we can change the assumption
$\Phi'(0)<\infty$ as
$\Phi'(0)=1$  without loss of generality. In fact, we can replace $\Phi(s)$ with $\frac{\Phi(s)}{\Phi'(0)}$ to met this assumption, because
the resulting $\Phi$ is still Bernstein and satisfies $\Phi(0)=0$, $\liml_{s \to \infty} \frac{\Phi(s)}{s} =0$
and $\Phi'(0) = 1$.

\begin{theorem} \label{thm:lp2}
Assume that the conditions in Theorem~\ref{thm:lapexp00} hold. Let $\Phi_{\alpha} (|b|)=\frac{\Phi(\alpha |b|)}{\Phi(\alpha)}$ for $\alpha>0$. If
$\Phi'(0) =  1$, then
\[
\lim_{\alpha \to 0+}  \Phi_{\alpha} (|b|) = |b| \quad \mbox{and} \quad \lim_{\alpha \to 0+}  \zeta(\Phi_{\alpha}) =0.
\]
Furthermore, 
for $b\neq 0$ we have that
\[
\lim_{\alpha \to \infty}   \Phi_{\alpha} (|b|) = |b|^{\gamma} \quad \mbox{and} \quad \lim_{\alpha \to \infty}  \zeta(\Phi_{\alpha}) =\infty.
\]
Here $\gamma=\liml_{s \to \infty} \frac{{s \Phi'(s)}}{\Phi(s)} \in [0, 1]$.
Especially, if $\gamma \in (0, 1)$, we also have that
\[
 \lim_{\alpha \to \infty} \frac{ \Phi'(\alpha |b|)}{\Phi'(\alpha)} =|b|^{\gamma{-}1}.
\]
\end{theorem}

\paragraph{Remarks 1} \;     It is worth noting that  $\Phi'$ is completely monotone on $[0, \infty)$.
Moreover,  $\Phi'$ is the Laplace transform of some probability distribution due to $\Phi'(0)=1$~\citep{FellerBook:1971}.
Additionally, Lemma~\ref{lem:lapexp} in the appendix shows that $\liml_{s \to \infty} \frac{{s \Phi'(s)}}{\Phi(s)}= 0$
whenever $\liml_{s \to \infty} \Phi(s)< \infty$. Notice that we cannot ensure that $\gamma<1$. For example,  it is known that the function $\Phi(s) \triangleq \frac{ s}{\log(e + s)}$ is a Bernstein function on $(0, \infty)$~\citep{SSVBook:2010}.
It is directly computed that $\liml_{s\to 0+ } \Phi(s)=0$, $\liml_{s\to 0+ } \Phi'(s)= 1$, and $\liml_{s\to 0+ } \Phi''(s)=-\frac{2}{e}$.
However, it is also obtained that $\liml_{s \to \infty} \frac{{s \Phi'(s)}}{\Phi(s)}=1$.

\paragraph{Remarks 2} \;
It follows from Theorem~1 in Chapter VIII.9 of \citet{FellerBook:1971} that $\liml_{s \to \infty} \frac{{s \Phi'(s)}}{\Phi(s)}= \gamma \in (0, 1)$ if and only if
$ \liml_{\alpha \to \infty} \frac{ \Phi'(\alpha |b|)}{\Phi'(\alpha)} =|b|^{\gamma{-}1}$.
However, $\liml_{\alpha \to \infty} \frac{ \Phi'(\alpha |b|)}{\Phi'(\alpha)} =|b|^{{-}1}$ (i.e., $\gamma=0$) is only sufficient for $\liml_{s \to \infty} \frac{{s \Phi'(s)}}{\Phi(s)}=0$.

\paragraph{Remarks 3} \; It is direct to obtain $\zeta(\Phi_{\alpha}) = - \frac{\alpha^2}{\Phi(\alpha)} \Phi''(0)$. 
Clearly, $\zeta(\Phi_{\alpha}) $ is increasing in $\alpha$. Thus, $\alpha$ controls the maximum concavity of $\Phi_{\alpha}$.

The second part of Theorem~\ref{thm:lp2} shows that the property of regular variation for the Bernstein function $\Phi(s)$ and its
derivative $\Phi'(s)$~\citep{FellerBook:1971}. That is, $\Phi$ and $\Phi'$ vary regularly with exponents  $\gamma$ and $\gamma{-}1$,
respectively. If  $\gamma=0$, then $\Phi$  varies slowly.
This property
implies an important connection of the Bernstein function with the $\ell_0$-norm and $\ell_1$-norm. With this connection,
we see that $\alpha$ plays a role of sparsity parameter because it measures sparseness of $\Phi(\alpha |\b|)/\Phi(\alpha)$.
In the following we present a  family of
Bernstein functions which admit the properties in Theorem~\ref{thm:lp2}.

\begin{table}[!ht]
\begin{center}
\caption{Several Bernstein  functions $\Phi_{\rho}(s)$ on $[0, \infty)$ as well as their  derivatives} \label{tab:exam}
\begin{tabular}{llll}
  \hline
 & Bernstein functions & First-order derivatives &  L\'{e}vy measures \\ \hline
KEP & $\Phi_{-1}(s) =  \sqrt{2 s {+}1} -1 $ & $\Phi_{-1}'(s)=  \frac{1}{\sqrt{2  s +1}}$
& $\nu(du) = \frac{1}{ \sqrt{2\pi}} u^{{-}\frac{3}{2}}  \exp({-} \frac{u}{2}) d u$\\
LOG &  $\Phi_{0}(s)=  \log\big(s  {+}1 \big)$ & $\Phi_{0}'(s)=  \frac{1}{s  {+}1}$ &
  $\nu(du) = \frac{1}{ u} \exp( {-} {u}) d u$ \\
LFR & $\Phi_{{1}/{2}}(s) =  \frac{2 s}{ s +2}$   & $\Phi'_{{1}/{2}}(s) =  \frac{4 }{({s}+{2})^2}$
&  $\nu(d u) = {4} \exp(- {2 u} ) d u$ \\
EXP & $\Phi_{1}(s) =  1- \exp(-  s)$ & $\Phi_{1}'(s) =   \exp(- s)$ & $\nu(d u) =  \delta_{1}(u) d u$ \\
\hline
\end{tabular}
\end{center}
\end{table}

\subsection{Examples}
\label{sec:example}

We consider a family of
Bernstein functions of the form
\begin{equation} \label{eqn:first}
\Phi_{\rho}(s) = \left\{\begin{array}{ll} \log(1+ s) & \mbox{if } \rho=0, \\
  \frac{1}{\rho}\Big[1- \big(1+ {(1{-}\rho)} s \big)^{-\frac{\rho}{1-\rho}} \Big]  & \mbox{if } \rho <1 \mbox{ and } \rho\neq 0, \\  1- \exp(- s) & \mbox{if } \rho=1. \end{array}
  \right.
\end{equation}
It is worth noting that this function is related to the Box-Cox transformation~\citep{BoxCoxAnalysis}.
It can be directly verified that $\Phi_{0}(s) = \liml_{\rho\to 0} \Phi_{\rho}(s)$ and $\Phi_{1}(s) = \liml_{\rho\to 1-} \Phi_{\rho}(s)$.
The corresponding L\'{e}vy measure is
\begin{equation} \label{eqn:first_nu}
\nu(d u) =  \frac{((1{-}\rho))^{-1/(1{-}\rho)}}{\Gamma(1/(1{-}\rho))} u^{\frac{\rho}{1-\rho}-1} \exp\Big( {-} \frac{u}{(1{-}\rho) } \Big) d u.
\end{equation}
Notice that  ${u}  \nu(d u)$ forms a Gamma measure for random variable $u$. Thus,
this L\'{e}vy measure $\nu(d u)$ is referred to as a generalized Gamma measure~\citep{Brix:1999}.
This family of the Bernstein functions were studied by \citet{Aalen:1992} for survival analysis.
We here show that they can be also used for sparsity modeling.

It is easily seen that the Bernstein functions $\Phi_{\rho}$ for $\rho \leq  1$ satisfy the conditions:
$\Phi(0)=0$,
$\Phi'(0) =1$ and $(-1)^{k}\Phi^{(k+1)}(0)< \infty$ for $k \in \NB$, in Theorem \ref{thm:lp2} and Lemma~\ref{lem:lapexp} (see the appendix). Thus,
$\Phi_{\rho}$ for  $\rho  \leq  1$ have the properties given in Theorem \ref{thm:lp2} and Lemma~\ref{lem:lapexp}.
These properties show that when letting $s=|b|$, the Bernstein functions $\Phi(|b|)$ form nonconvex penalty functions.

The derivative of $\Phi_{\rho}(s)$ is defined by
\begin{equation} \label{eqn:deriv}
\Phi'_{\rho}(s) = \left\{\begin{array}{ll} \frac{1}{1 {+} s} & \mbox{if } \rho=0, \\
\big(1+ {(1{-}\rho)} s \big)^{-\frac{1}{1-\rho}}  & \mbox{if } \rho <1 \mbox{ and } \rho\neq 0, \\  \exp(- s) & \mbox{if } \rho=1. \end{array}
  \right.
\end{equation}
It  is also directly verified that $\Phi'_{0}(s) = \liml_{\rho\to 0} \Phi'_{\rho}(s)$ and $\Phi'_{1}(s) = \liml_{\rho\to 1-} \Phi'_{\rho}(s)$.
When $\rho \in [0, 1]$, we have $\liml_{s \to \infty} \frac{s \Phi'_{\rho}(s)}{\Phi_{\rho}(s)} =0$ (or $\liml_{s\rightarrow \infty} \frac{\Phi_{\rho}(s)}{\log(s)}<\infty$).
When  $\rho < 0$, we then have $\liml_{s \to \infty} \frac{s \Phi'_{\rho}(s)}{\Phi_{\rho}(s)} = \frac{\rho}{\rho{-}1} \in (0, 1)$.

\begin{proposition}\label{pro:33} Let $\Phi_{\rho}(s)$ on $[0, \infty)$ be defined in (\ref{eqn:first}) and $\Phi_{\rho, \alpha} (s)=\frac{\Phi_{\rho}(\alpha s)}{\Phi_{\rho}(\alpha )}  $ for $\alpha>0$. Then
\begin{enumerate}
\item[\emph{(a)}] If $-\infty<\rho_1< \rho_2\leq 1$ then $\Phi'_{\rho_1}(s)\geq \Phi'_{\rho_2}(s)$,  $\Phi_{\rho_1}(s)\geq \Phi_{\rho_2}(s)$, and $\zeta(\Phi_{\rho_1, \alpha}) \leq \zeta(\Phi_{\rho_2, \alpha} ) $;
\item[\emph{(b)}] $\liml_{\alpha \to \infty} \frac{ \Phi'_{\rho}( \alpha)}{\alpha^{\gamma-1}} = (1-\gamma)^{1-\gamma}$ where $\gamma=0$ if $\rho \in (0, 1]$ and $\gamma=\frac{\rho}{\rho{-}1}$ if $\rho \in (-\infty, 0]$; and
\[
\liml_{\alpha \to \infty}  \frac{\Phi_{\rho}(\alpha s)}{\Phi_{\rho}(\alpha )} = \left\{\begin{array}{ll} 1 & \mbox{if } \rho \in [0, 1], \\
s^{\frac{\rho}{\rho{-}1}} & \mbox{if } \rho \in (-\infty, 0). \end{array} \right.
\]
\end{enumerate}
\end{proposition}

Proposition~\ref{pro:33}-(b) shows the property of regular variation for $\Phi_{\rho}$; that is, $\Phi_{\rho}$ varies slowly when $0 \leq \rho \leq 1$, while it
varies regularly with exponent $\rho/(\rho{-}1)$ when $\rho < 0$.
Thus, $\frac{\Phi_{\rho}(\alpha |b|)}{\Phi_{\rho}(\alpha )}$ for $\rho < 0$ approaches to the $\ell_{\rho/(\rho{-}1)}$-norm $\|b\|_{\rho/(\rho{-}1)}$
as $\alpha \to \infty$.

We list four special Bernstein functions in Table~\ref{tab:exam} by taking different $\rho$. Specifically,
these  penalties are  the kinetic energy plus (KEP) function, nonconvex \emph{log-penalty} (LOG),
nonconvex \emph{exponential-penalty} (EXP), and \emph{linear-fractional} (LFR) function, respectively. Figure~\ref{fig:berns} depicts these functions and their derivatives.
In Table~\ref{tab:exam} we also give the L\'{e}vy measures  corresponding to these functions.
Clearly, KEP gets a continuum of penalties from $\ell_{1/2}$ to the $\ell_1$, as varying $\alpha$
from $\infty$ to $0$~\citep{ZhangTR:2013}. But the LOG, EXP and LFR functions get the entire continuum of functions from $\ell_{0}$ to the $\ell_1$.

The LOG, EXP and LFR functions  have been applied in the literature~\citep{BradleyICML:1998,GaoAAAI:2011,WestonJMLR:2003,GemanPAMI:1992,NikolovaSIAM:2005,LvFan:2009}. \citet{LvFan:2009} also called LFR a smooth integration of counting and absolute deviation (SICA) penalty.  
In image processing and computer vision, these functions are usually also called \emph{potential functions}.
However, to the best of our  knowledge, there is no work to establish their connection with Bernstein functions.

\begin{figure}[!ht]
\centering
\subfigure[$\Phi_{\rho}(s)$ ]{\includegraphics[width=75mm,height=50mm]{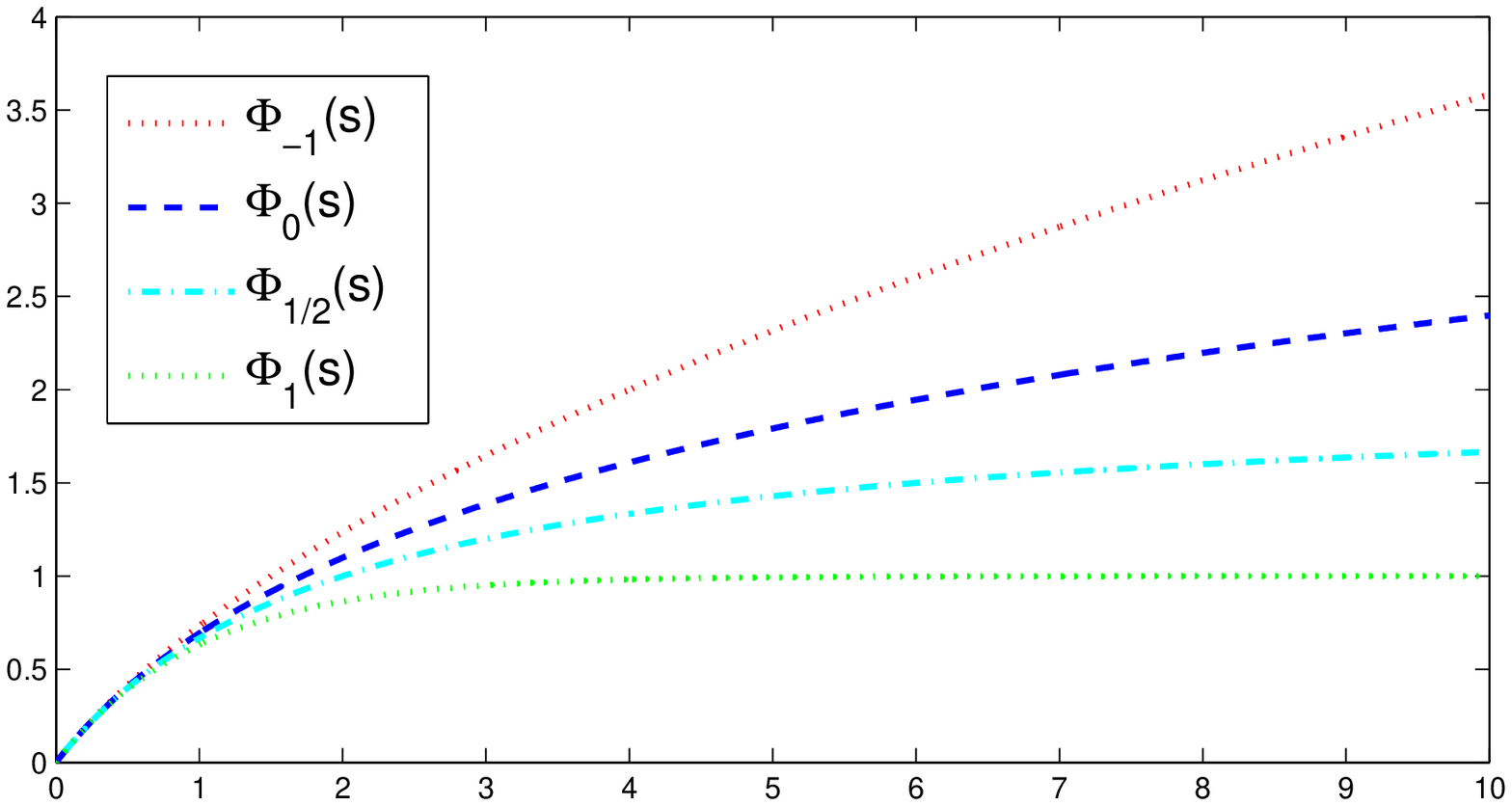}} \subfigure[$\Phi'_{\rho}(s)$]{\includegraphics[width=75mm,height=50mm]{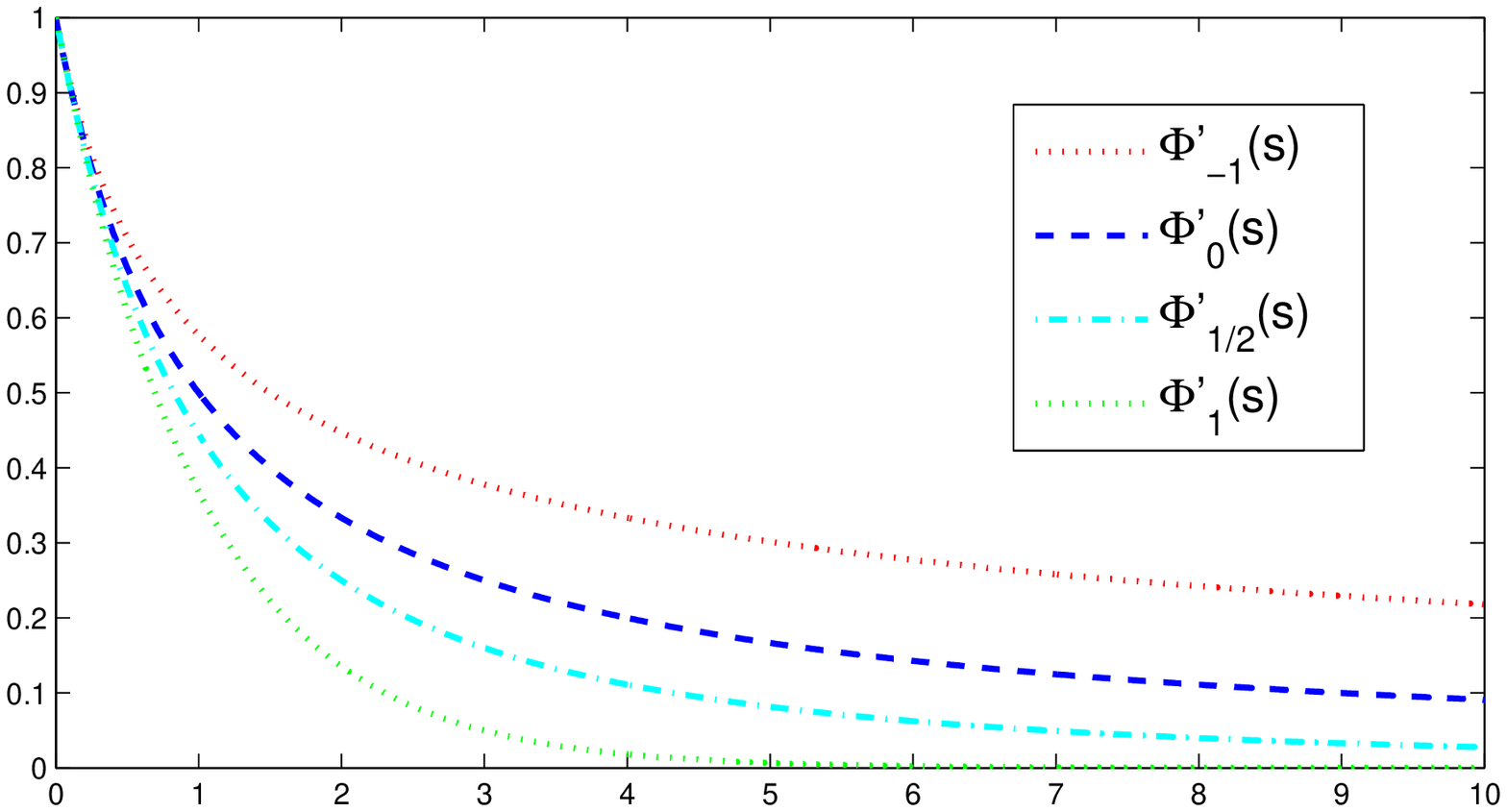}} \\
\caption{(a) The Bernstein functions $\Phi_{\rho}$ for $\rho=-1$, $\rho=0$, $\rho=\frac{1}{2}$ and $\rho=1$  corresponding to KEP, LOG, LFR and EXP.  (b) The corresponding derivatives $\Phi'_{\rho}$.}
\label{fig:berns}
\end{figure}

\begin{figure}[!ht]
\centering
 \subfigure[]{\includegraphics[width=100mm,height=70mm]{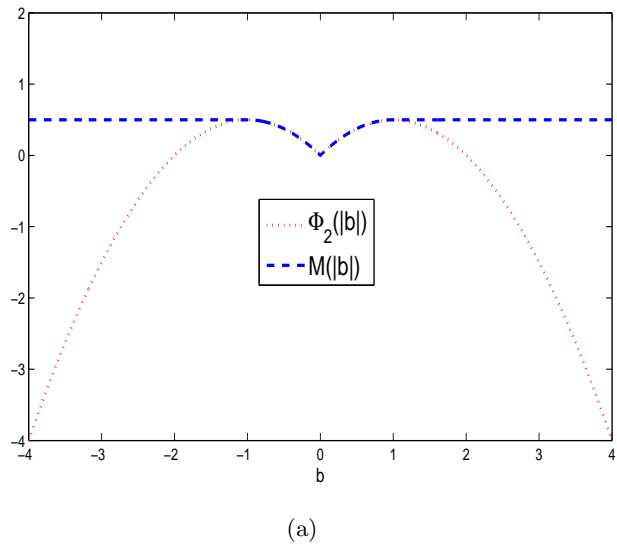}} \\
\caption{The Bernstein function $\Phi_{2}(|b|)$ and the MCP function $M(|b|)$.}
\label{fig:berns2}
\end{figure}

Finally, we note that the MCP function can be regarded as a capped version of  $\Phi_{2}$ (i.e.,  $\rho=2$).
Clearly,  $\Phi_{2}(s)$ is well-defined for $s\geq 0$ but no longer Bernstein, because $\Phi_{2}(s)$ is negative when $s> 2$.
Moreover,
it is decreasing when $s\geq 1$ (see Figure~\ref{fig:berns2}).  To make a concave penalty function from $\Phi_{2}$, we truncate
$\Phi_{2}(s)$ into $1/2$ whenever $s\geq 1$,  yielding the MCP function given in (\ref{eqn:mcp}).

\section{Sparse Estimation Based on Bernstein Penalty Functions}
\label{sec:sest}

We now study  mathematical  properties of  the sparse estimators based on Bernstein penalty functions.
These properties show that Bernstein penalty functions are suitable for use of a coordinate descent algorithm.

\subsection{Thresholding Operators}
\label{sec:threshold}

Let $\Phi$ be a Bernstein penalty function.
We  define a univariate penalized least squares problem as follows.
\begin{equation} \label{eqn:general}
J_{1}(b)\triangleq \frac{1}{2} (z- b)^2 +  {\lambda} \Phi(|b|),
\end{equation}
where $z=\x^T\y$. It has been established by \citet{Fan01}  that a good penalty should result in an estimator with three properties.
(a) ``Unbiasedness:" it is nearly unbiased when the true unknown parameter takes a large value in magnitude; (b) ``Sparsity:" there is a thresholding operator, which
automatically sets small estimated coefficients to zero; (c) ``Continuity:" it is continuous in $z$, which can avoid instability
in model computation and prediction.

It suffices for
the estimator obtained from (\ref{eqn:general}) to be nearly unbiased that $\Phi'(|b|) \rightarrow 0$ as $|b|\rightarrow \infty$.
The Bernstein penalty function satisfies the conditions
$\Phi(0+)=0$ and
$\liml_{s \to \infty} \Phi'(s) = 0$, so it can result in an unbiased sparse estimator.

\begin{theorem} \label{thm:sparsty} Let $\Phi$ be a nonzero Bernstein function  on $[0, \infty)$ such that
$\Phi(0)= 0$  and $\liml_{s \to \infty} \frac{\Phi(s)}{s}=\liml_{s \to \infty} \Phi'(s) = 0$.
Consider the penalized least squares problem in (\ref{eqn:general}).
\begin{enumerate}
\item[\emph{(i)}] If $\lambda \leq -\frac{1}{\Phi{''}(0)}$, then the resulting estimator  is defined as
\[
\hat{b} = S(z, \lambda) \triangleq  \left\{ \begin{array}{ll}
{{\sgn}(z)} \kappa(|z|) & \textrm{ if } |z| >
{\lambda} \Phi{'}(0), \\
0 & \textrm{ if } |z| \leq {\lambda}\Phi{'}(0),
\end{array} \right.
\]
where $\kappa(|z|)\in (0, |z|)$ is the unique positive root of
$b {+} {\lambda} \Phi'(b) {-} |z|=0$ in $b$.
\item[\emph{(ii)}] If $\lambda > -\frac{1}{\Phi{''}(0)}$, then the resulting estimator  is defined as
\[
\hat{b} = S(z, \lambda) \triangleq  \left\{ \begin{array}{ll}
{{\sgn}(z)} \kappa(|z|) & \textrm{ if } |z| >
s^* + {\lambda} \Phi'(s^*), \\
0 & \textrm{ if } |z|\leq s^* + {\lambda} \Phi'(s^*),
\end{array} \right.
\]
where $s^*>0$ is the unique root of $1 {+} {\lambda} \Phi{''}(s)=0$ and $\kappa(|z|)$
is the unique root of $b {+} {\lambda} \Phi'(b) {-} |z|=0$ on $(s^*, |z|)$.
\end{enumerate}
\end{theorem}

As we see earlier, we always have $\Phi'(0)>0$ and $\Phi{''}(0)<0$.
It is worth noting that when $\lambda \leq -\frac{1}{\Phi{''}(0)}$ the function $h(b)\triangleq b+ {\lambda} \Phi'(b) - |z|$
is increasing on $(0, |z|)$ and that when $\lambda > -\frac{1}{\Phi{''}(0)}$ it
is also increasing on $(s^*, |z|)$. Thus, we can employ the bisection method
to find the corresponding root $\kappa(|z|)$. 
We will see that an analytic solution for $\kappa(|z|)$
is available when $\Phi(s)$ is either of KEP, LOG and LFR. Therefore,
a coordinate descent algorithm is especially appropriate for
Bernstein penalty functions, which will be presented in Section~\ref{sec:cda}.

It suffices for the resulting estimator to be  sparse  that the minimum of
the function $|b|+ \lambda \Phi'(|b|)$ is positive. Moreover, a sufficient and necessary condition for ``continuity" is that
the minimum of
$|b|+ \lambda \Phi'(|b|)$ is attained at $0$. In our case, it follows from the proof of Theorem \ref{thm:sparsty}
that when $\lambda \leq -\frac{1}{\Phi{''}(0)}$, $|b|+ \lambda \Phi'(|b|)$ attains its minimum value ${\lambda}\Phi'(0)$ at $s^{*}=0$. Thus,
the resulting estimator  is sparse and continuous when $\lambda \leq -\frac{1}{\Phi{''}(0)}$.
In fact, the continuity can be also concluded directly from Theorem~\ref{thm:sparsty}-(i).
Specifically, when  $\lambda \leq -\frac{1}{\Phi{''}(0)}$,
we have $\kappa(\lambda \Phi'(0)) =0$ because  $0$ is the unique root of equation $b {+} {\lambda} \Phi'(b) {-} {\lambda} \Phi'(0)=0$.

Recall that if $\Phi(s)=s^{q}$ with $q \in (0, 1)$, we have $\liml_{s \to 0} \Phi'(s)=+\infty$  and
$\liml_{s \to 0} \Phi{''}(s)=-\infty$. This implies that $\lambda \leq -\frac{1}{\Phi{''}(0)}$
does  not hold. In other words, this  penalty  cannot result in a continuous solution.

In this paper we are especially concerned with the Bernstein penalty functions which satisfy the conditions
in Theorem~\ref{thm:lp2}.
In this case, since $-\infty<\Phi{''}(0)<0$ and $0<\Phi'(0)<\infty$,
such Bernstein penalties are able to result in a continuous sparse
solution. Consider the regular variation property of $\Phi(s)$ given in Theorem~\ref{thm:lp2}. We
let  $P(b) = {\Phi(\alpha |b|)}$ and $\lambda=\frac{\eta}{\Phi(\alpha)}$
where  $\eta$ and $\alpha$ are positive constants.
We now denote the thresholding operator  $S(z, \lambda)$ in Theorem~\ref{thm:sparsty} by $S_{\alpha}(z, \eta)$. As a direct corollary of Theorem~\ref{thm:sparsty}, we particularly have the following results.

\begin{corollary} \label{cor:threshold} Assume $\Phi'(0) =1$ and $\Phi{''}(0)>-\infty$. Let $P(b) = {\Phi(\alpha |b|)}$ and $\lambda=\frac{\eta}{\Phi(\alpha)}$ where $\alpha>0$ and $\eta>0$,
and let $S_{\alpha}(z, \eta)$ be the thresholding operator defined in Theorem~\ref{thm:sparsty}.
\begin{enumerate}
\item[\emph{(i)}] If $\eta \leq -\frac{\Phi(\alpha)}{\alpha^2 \Phi{''}(0)}$, then the resulting estimator  is defined as
\[
\hat{b} = S_{\alpha}(z, \eta) \triangleq  \left\{ \begin{array}{ll}
{{\sgn}(z)} \kappa(|z|) & \textrm{ if } |z| >
\frac{\alpha}{\Phi(\alpha)} \eta, \\
0 & \textrm{ if } |z| \leq \frac{\alpha}{\Phi(\alpha)} \eta,
\end{array} \right.
\]
where $\kappa(|z|)\in (0, |z|)$ is the unique positive root of
$b+ \frac{\eta \alpha}{\Phi(\alpha)} \Phi'(\alpha b) - |z|=0$ w.r.t.  $b$.
\item[\emph{(ii)}] If $\eta > - \frac{\Phi(\alpha)}{\alpha^2 \Phi{''}(0)}$, then the resulting estimator  is defined as
\[
\hat{b} = S_{\alpha}(z, \eta)  \triangleq  \left\{ \begin{array}{ll}
{{\sgn}(z)} \kappa(|z|) & \textrm{ if } |z| >
s^* + \frac{\alpha \Phi'(\alpha s^*)}{\Phi(\alpha)}  \eta, \\
0 & \textrm{ if } |z|\leq s^* + \frac{\alpha \Phi'(\alpha s^*)}{\Phi(\alpha)}  \eta,
\end{array} \right.
\]
where $s^*>0$ is the unique root of $1+ \frac{\eta  \alpha^2}{\Phi(\alpha)} \Phi{''}(\alpha s)=0$ and $\kappa(|z|)$ is the unique root of the equation $b+ \frac{\eta \alpha}{\Phi(\alpha)}   \Phi'(\alpha b) - |z|=0$ on $(s^*, |z|)$.
\end{enumerate}
\end{corollary}

\begin{proposition} \label{thm:nest} Assume  $\Phi'(0) =1$ and $\Phi{''}(0)>-\infty$.
Then
\begin{enumerate}
\item[\emph{(a)}]   $\frac{\alpha}{\Phi(\alpha)}$ is increasing  and $\frac{1}{\Phi(\alpha)}$ is decreasing, both in $\alpha$ on $(0, \infty)$. Moreover, $\frac{\alpha}{\Phi(\alpha)}> 1$, $\liml_{\alpha \to 0+} \frac{\alpha }{\Phi(\alpha)}=1$ and $\liml_{\alpha \to \infty} \frac{\alpha }{\Phi(\alpha)} = \infty$.
\item[\emph{(b)}] The root $\kappa(|z|)$ is strictly   increasing w.r.t.\ $|z|$.
\item[\emph{(c)}] When $\lambda< - 1/\Phi{''}(0)$, the function $\frac{1}{2} b^2 + \lambda \Phi(|b|)$ is strictly convex in $b \in \RB$. 
\end{enumerate}
\end{proposition}

The Bernstein function $\Phi_{\rho}$ given in (\ref{eqn:first}) satisfies the conditions in Corollary~\ref{cor:threshold} and Proposition~\ref{thm:nest}. Proposition~\ref{thm:nest}-(a) and (c) implies that $\Phi(|b|)$ satisfies Assumption~1 made in 
\citet{LohWainwright:2013}.  
Recall that $\alpha$ controls sparseness of $\Phi(\alpha|b|)/\Phi(\alpha)$ as varying $\alpha$ from 0 to $\infty$.
It follows from Proposition~\ref{thm:nest} that $|z|\geq \eta$ due to $|z|\geq \frac{\eta \alpha}{\Phi(\alpha)}$. This implies that
the Bernstein function $\Phi(\alpha|b|)/\Phi(\alpha)$
has stronger sparseness  than the $\ell_1$-norm when $\eta \leq -\frac{\Phi(\alpha)}{\alpha^2 \Phi{''}(0)}$. Moreover, for a fixed $\eta$,
there is a strict nesting of
the shrinkage thresholding $\frac{\eta \alpha}{\Phi(\alpha)}$  as $\alpha$ increases.   Thus,
the Bernstein penalty to some extent satisfies the nesting property, a
desirable property for thresholding functions pointed out by~\citet{MazumderSparsenet:11}.

As we stated earlier,
when $\rho \in[0, 1]$ $\Phi_{\rho}$ gives a smooth homotopy between the $\ell_0$-norm and the $\ell_1$-norm. 
We now explore a connection of the thresholding operator  $S_{\alpha}(z, \eta)$ with the soft thresholding operator
based on the lasso
and the hard thresholding operator based on the $\ell_0$-norm.

\begin{theorem} \label{thm:0-1}
Let  $S_{\alpha}(z, \eta)$ be the thresholding operator defined in Corollary~\ref{cor:threshold}. Then
\[
\lim_{\alpha \to 0+}  S_{\alpha}(z, \eta) = \left\{ \begin{array}{ll}
{\sgn}(z) (|z| - \eta) & \textrm{ if } |z| >
{\eta}, \\
0 & \textrm{ if } |z| \leq {\eta}.
\end{array} \right.
\]
Furthermore, if $\liml_{\alpha \to \infty}  \frac{\alpha \Phi'(\alpha)} {\Phi(\alpha)} =0$  or
$\liml_{\alpha \to \infty} \frac{\Phi(\alpha)}{\log(\alpha)} < \infty$,
then
\[
\lim_{\alpha \to \infty}  S_{\alpha}(z, \eta) = \left\{ \begin{array}{ll}
z & \textrm{ if } |z| >
0, \\
0 & \textrm{ if } |z| \leq 0.
\end{array} \right.
\]
\end{theorem}

In the limiting case of $\alpha \to 0$, Theorem~\ref{thm:0-1} shows that the thresholding function $S_{\alpha}(z, \eta)$
approaches the soft thresholding function
$\sgn(z) (|z|-\eta)_{+}$. However,
as $\alpha \to \infty$, the limiting  solution does not fully agree with the hard thresholding function, which is defined as
$z \mathrm{I} (|z| \geq \sqrt{2 \eta})$.

Let us return the concrete Bernstein functions in Table~\ref{tab:exam}.
We are especially
interested in the KEP, LOG and LFR functions, because
there are analytic solutions for $\kappa(|z|)$ based on them.
Corresponding to  LOG and LFR,  $\kappa(|z|)$ are respectively
\begin{equation} \label{eqn:logan}
\kappa(|z|) =\frac{\alpha|z|-1 + \sqrt{(1+\alpha |z|)^2- {4 \lambda \alpha^2} }}{2 \alpha}
\end{equation}
  and
\begin{equation} \label{eqn:lfran}
\kappa(|z|) = \frac{2(\alpha |z| {+} 2)}{3 {\alpha}}  \cos\Big[\frac{1}{3}
\arccos\big( 1{-}{\lambda \alpha^2 } (\frac{3}{\alpha|z| {+} 2})^{3}  \big) \Big] {+} \frac{\alpha |z| {+} 2}{3 {\alpha}} {-}
\frac{2}{\alpha}.
\end{equation}
The derivation can be obtained by using direct algebraic computations.
We here omit the derivation details.
As for KEP, $\kappa(|z|)$ was derived by
\citet{ZhangTR:2013}.  That is,
\[
\kappa(|z|) = \frac{4(2\alpha |z| {+} 1)}{3}\cos^2\Big[\frac{1}{3 \alpha}
\arccos\big( {-} { \lambda \alpha^2}  (\frac{3}{2\alpha|z| {+} 1})^{\frac{3}{2}}  \big) \Big] {-} \frac{1}{\alpha}.
\]

\subsection{Coordinate Descent Algorithm for the Penalized Regression Problem}
\label{sec:cda}

Based on the discussion in the previous subsection, the Bernstein penalty function is suitable
for the coordinate descent algorithm under the regression setting where $L(\b; \X, \y) = \frac{1}{2}\|\y -\X \b \|_2^2$. Specifically,
we give the coordinate descent procedure in Algorithm~\ref{alg:coord}.
If the LOG and LFR functions are used, the corresponding  thresholding  operators have the analytic forms in (\ref{eqn:logan})
and (\ref{eqn:lfran}).
Otherwise, we employ the bisection method for finding the root $\kappa(|z|)$.
The method is also very efficient.

When $\lambda \leq -\frac{1}{\alpha^2 \Phi{''}(0)}$ (or $\lambda > -\frac{1}{\alpha^2\Phi{''}(0)}$), we can obtain
that  $|\hat{b}|\leq |z|$ always holds. The
objective function $J_{1}(b)$ in (\ref{eqn:general}) is strictly convex in $b$ whenever $\lambda \leq -\frac{1}{\alpha^2 \Phi{''}(0)}$.
Moreover, according to Theorem~\ref{thm:nest}, the estimator
$\hat{b}$ in both the cases is strictly increasing w.r.t.\ $|z|$.
As we see, $P(b)\triangleq  \Phi(\alpha |b|)$ satisfies $P(b)=P(-b)$. Moreover, $P'(b)$ is positive and uniformly
bounded  on $[0, \infty)$, and $\inf_{b} \;  \lambda P{''}(b) >-1$ on $[0, \infty)$ when $\lambda < - \frac{1}{\alpha^2 \Phi{''}(0)}$. Thus,
the algorithm shares the same  convergence property as in \citet{MazumderSparsenet:11} (see Theorem 4 therein).
It is worth noting that Theorem 4 of \citet{MazumderSparsenet:11} requires the second-order derivative $P{''}(b)$ on $[0, \infty)$ to exist.
However,  for the MCP function $M$ defined in (\ref{eqn:mcp}) its second-order derivative at $|b|=1$ does not exist.
In contrast,
our used Bernstein penalty function meets this requirement.
In Section~\ref{sec:convergence} we will give a global  convergence analysis based on the Kurdyka-\L{}ojasiewicz property.

\begin{algorithm}[!ht]
   \caption{Coordinate descent algorithm for the penalized regression problem}
   \label{alg:coord}
\begin{algorithmic}
   \STATE {\bfseries Input:} $\{\x_i, y_i\}_{i=1}^n$ where each column of $\X=[\x_i, \ldots, \x_n]^T$ is standardized
   to have mean 0 and length 1,
   a grid of increasing values $\Lambda=\{\eta_1, \ldots, \eta_L\}$, a grid of decreasing values $\Gamma=\{\alpha_1, \ldots, \alpha_K\}$
   where $\alpha_K$ indexes the Lasso penalty. Set $\hat{\b}_{\alpha_K, \eta_{L+1}}=0$.
   \FOR{each value of $l \in \{L, L-1, \ldots, 1\}$ }
   \STATE Initialize  $\tilde{\b} = \hat{\b}_{\alpha_K, \eta_{l+1}}$;
    \FOR{each value of $k \in \{K, K-1, \ldots, 1\}$ }
    \IF{$\eta_l \leq - \frac{\Phi(\alpha_k)}{\alpha_k^2 \Phi{''}(0)}$ }
     \STATE Cycle through the following one-at-a-time updates
      \[\tilde{b}_j = S_{\alpha_{k}}\Big(\sum_{i=1}^n(y_i- z_{i}^j)x_{ij}, \eta_l\Big), \quad j=1, \ldots, p
       \]
      where $z_i^j=\sum_{k\neq j} x_{ik} \tilde{b}_k$, until the updates converge to $\b^{\ast}$;
   \STATE $\hat{\b}_{\alpha_k, \eta_l} \leftarrow \b^{\ast}$.
   \ENDIF
   \ENDFOR
   \STATE Increment $k$;
   \ENDFOR
   \STATE Decrement $l$;
   \STATE {\bfseries Output:} Return the two-dimensional solution $\hat{\b}_{\alpha, \eta}$ for $(\alpha, \eta) \in \Lambda{\times}\Gamma$.
\end{algorithmic}
\end{algorithm}

\subsection{Extension to Classification and Robust Regression  Problems}

We now consider the classification problem in which the loss function is defined as $L(\b;  \X, \y)= \sum_{i=1}^n \log(1+ \exp(-y_i \b^T \x_i))$ and the penalty function is still defined as  Bernstein function $\Phi(\alpha |b|)$.
\citet{BrehenyAAS:2010} suggested that the corresponding minimization problem is approached by first obtaining a quadratic approximation to the loss function $L(\b;  \X, \y)$ based on a Taylor series expansion about the current iterative value of  $\b$.
That is,
\begin{align*}
\b^{(t+1)} = & \argmin_{\b} \Big\{L(\b^{(t)};  \X, \y) + \langle \nabla L(\b^{(t)};  \X, \y), \b -\b^{(t)} \rangle  \\
& \quad \quad + \frac{1}{2}(\b -\b^{(t)})^T \nabla^2
L(\b^{(t)};  \X, \y) (\b -\b^{(t)}) + \lambda_n \sum_{j=1}^p \Phi(\alpha |b_j|)  \Big\}.
\end{align*}

Alternatively, we resort to a proximal alternating linearized minimization (PALM) procedure to
solve the minimization problem.
Specifically,  the procedure chooses variables $b_1, \cdots, b_p$ in cyclic order at each time. Let $L_j^{(t)}(b_j)=L(\b_{-j}^{(t)}; \X, \y)$
where $\b_{-j}^{(t)}=(b_1^{(t{+}1)}, \ldots, b_{j-1}^{(t{+}1)}, b_j, b_{j+1}^{(t)}, \ldots, b_{p}^{(t)})^T$.
When optimizing variable $b_j$ with the rest variables fixed, we use a linear approximation of $L_j$ with a proximal regularization term.
That is,
\begin{equation}\label{eq:app}
 b_j^{(t+1)}=\argmin_{b_j} \Big\{L^{(t)}_j(b_j^{(t)}) +  \nabla_j L_j^{(t)}(b_j^{(t)}) (b_j - b_j^{(t)}) +
 \frac{\nu_j^{(t)}}{2} (b_j - b_j^{(t)})^2 + \lambda_n  \Phi(\alpha |b_j|)\Big\}.
\end{equation}
Typically, the optimal solution can be represented as $b^{(t+1)}_{j}=\mathsf{Prox}_{\nu_j^{(t)}}^{P}\big(b_j^{(t)}-\frac{1}{\nu_j^{(t)}}\nabla L_j^{(t)}(b_j^{(t)}) \big)$.
The proximal operator is defined as
\[
   \mathsf{Prox}_{\nu}^{g}(u) \triangleq \argmin_{x} \Big\{\frac{\nu}{2}\|x-u\|^2+g(x)\Big\}.
\]
We summary the whole procedure in Algorithm~\ref{alg:palm}.

Algorithm~\ref{alg:palm} can be also used to solve a  robust regression problem regularized by the Bernstein function.  
The loss function is then defined by
\[
L(\b;  \X, \y)= \sum_{i=1}^n =  L_{\delta}(y_i- \b^T \x_i), 
\] 
where $L_{\delta}(y_i- \b^T \x_i)$ is the Huber loss  as  in \eqref{eqn:huber}.  

\begin{algorithm}[!ht]
   \caption{PALM  for the penalized classification problem}
   \label{alg:palm}
\begin{algorithmic}
   \STATE {\bfseries Initialization:} set the initial value $\b_{(0)}$.
    \FOR{$t =0, 1, \ldots$}
     \FOR{$j =1, 2, \ldots,  p$}
    \STATE
    \[
    b^{(t+1)}_{j} = \mathsf{Prox}_{\nu_j^{(t)}}^{P} \big(b_j^{(t)}-\frac{1}{\nu_j^{(t)}}\nabla L_j^{(t)}(b_j^{(t)}) \big).
    \]
      \ENDFOR
   \IF{ stopping criterion is met}
      \STATE  Return  $\b^{(t)}$.
    \ENDIF
 \ENDFOR
\end{algorithmic}
\end{algorithm}

\section{Convergence Analysis}
\label{sec:convergence}

In this section we present the global convergence analysis of the previous coordinate descent  and  PALM procedures in Algorithms~\ref{alg:coord} and \ref{alg:palm}.
In particular, we consider the following optimization problem
\[
\min_{\b} \; F(\b) \triangleq L(\b; \X, \y) + \lambda_n \sum_{j=1}^p P(b_j),
\]
where  $P(b_j)= \Phi(\alpha |b_j|)$. We further make the following assumptions:

\begin{description}
\item[Assumption 1] In Algorithm~\ref{alg:palm}, assume that $0< m_0 < \nu_j^{(t)} < M_0< \infty$ for every $j$ and $t$.
\item[Assumption 2] $L_j^{(t)}$ is strongly convex with modulus $0<-\lambda_n \alpha^2 \Phi''(0) <\gamma_j^{(t)}< \gamma_M < \infty$, namely,
\[
L_j^{(t)}(\u) - L_j^{(t)}(\v) \geq  \langle \nabla L_j^{(t)}(\v), \u-\v \rangle + \frac{\gamma_j^{(t)}}{2} \|\u-\v\|^2;
\]
and $\nabla L_j^{(t)}$ is Lipschitz continuous. 
\end{description}
Notice that both the logistic loss function and the least squares loss function meet Assumption~2~\citep{XuYin}.
The Huber loss function also salsifies Assumption~2. Thus, the following convergence analysis applies to the case that the loss function is defined as  the Huber loss function in \eqref{eqn:huber}.

Our convergence analysis mainly  includes three steps.  First,   we show  the sequences $\{F(\b^{(t)}): t \in \NB \}$ generated by the algorithms have the sufficient decrease property,  and hence establish  the square summable result $\sum_{t=0}^{\infty}\|\b^{(t+1)}- \b^{(t)} \|^2< + \infty$  in Theorem~\ref{thm:decrease}. 
Second,  based on the fact that the Bernstein penalty function has the  Kurdyka-\L{}ojasiewicz property,   in Theorem~\ref{thm:global} we improve the result to   $\sum_{t=0}^{\infty}\|\b^{(t+1)}- \b^{(t)} \| < + \infty$, which implies that the generated sequence $\{\b^{(t)}: t \in \NB \}$ is a Cauchy sequence.  Consequently,  it converges to a critical point of the objective function $F$. Third, noting that the Bernstein penalty function satisfies Condition~1 of \cite{LvFan:2009}, as a direct corollary of Theorem~1 of  \cite{LvFan:2009,FanLv:2011}, we prove that  $\{\b^{(t)}: t \in \NB \}$ converges to a strict local minimizer of $F$ under certain regularity conditions.

\begin{theorem} \label{thm:decrease}
Suppose Assumption 1  holds for Algorithm~\ref{alg:palm}, or Assumption 2 holds for Algorithm~\ref{alg:coord}.
Let the sequence $\left\{\b^{(t)}: {t \in \mathbb N} \right\}$ be generated by Algorithm~\ref{alg:coord} or Algorithm~\ref{alg:palm}. Then we have the following properties:
\begin{enumerate}
\item[\emph{(i)}] \emph{\bf [Sufficient decrease property]} The generated sequence $\left\{F(\b^{(t)}): {t \in \mathbb N}\right\}$ is nonincreasing; particularly,
\[
F(\b^{(t)})- F(\b^{(t+1)}) \geq \frac{C_0}{2}\|\b^{(t)} - \b^{(t+1)}\|^2, \quad \forall \; t \geq 0,
    \]
where $C_0$ is some positive constant.
\item[\emph{(ii)}] \emph{\bf [Square summable  property]}
\[\sum_{t=0}^{\infty}\|\b^{(t+1)}- \b^{(t)} \|^2< + \infty, \]
which implies $\lim_{t \to\infty}\|\b^{(t+1)}- \b^{(t)}\|=0$.
\item[\emph{(iii)}]  \emph{\bf [Subgradient lower bound for the iterative gap]} There exists a positive constant $C_1$ such that for $\w^{(t+1)} \in \partial F(\b^{(t+1)})$, 
\begin{equation}  \label{eqn:subinequality}
\|\w^{(t+1)}\| \leq C_1 \| \b^{(t+1)} - \b^{(t)} \|, \; \forall \; t=0, 1, \ldots
\end{equation}
\end{enumerate}
\end{theorem}

Notice that the function $F(\b)$ is coercive, which means that
$F(\b)\to\infty$ iff $\|\b\|\to \infty$. 
Then by Theorem~\ref{thm:decrease}-(i), with a bounded initial   $\b^{(0)}$  the sequence  $\left\{\b^{(t)}: t \in \NB \right\}$ is bounded. 
Hence, there exists a convergent subsequence $\left\{\b^{(t_k)}\right\}$ that converges to $\b^*$. The set of all stationary points which are started with a bounded $\b^{(0)}$ is denoted by $\MM(\b^{(0})$. That is,
\[
\MM(\b^{(0)})\triangleq \Big\{{\bar \b} \in \RB^p: \exists  ~t_k,~\left\{t_k\right\}_{k \in \NB}, ~such ~that ~\b^{(t_k)}\to {\bar \b} ~as ~k ~\to ~\infty\Big\}.
\]

\begin{lemma}[property of the limit points]\label{lem:critical}
Let  $\left\{\b^{(t)}: t \in \NB \right\}$ be generated by Algorithm~\ref{alg:coord} or Algorithm~\ref{alg:palm}.  Then we have
\begin{enumerate}
\renewcommand{\labelenumi}{(\theenumi)}
\item[\emph{(i)}]  $\MM(\b^{(0)})$  is not empty and $\MM(\b^{(0)})$ $\subseteq \mathrm{crit} F$;
\item[\emph{(ii)}]  \begin{equation}\label{eq:set}
\lim_{t \to \infty} {\mathrm{dist}}\big(\b^{(t)}, \MM(\b^{(0)})\big)=0. 
\end{equation}
\end{enumerate}
\end{lemma}

Lemma \ref{lem:critical} implies that   ${\MM}(\b^{(0)})$ is a subset of stationary or critical points of $F$ and $\{\b^{(t)}\}_{t \in\NB}$  approaches to one point of ${\MM}(\b^{(0)})$.   Our current concern is to prove $\liml_{t \to\infty}\b^{(t)}=\b^*$.  As in \cite{BolteMathProgram}, we know that $\MM(\b^{(0)})$ is compact and connected. Moreover, the objective function $F$ is finite and constant on $\MM(\w^{(0)})$. 

As we have mentioned previously, as a function of $b$,
$\Phi(|b|)$  is sub-analytic, which satisfies the Kurdyka-\L{}ojasiewicz property. Moreover, both the least squares loss and
the logistic loss function are real analytic. 
This implies that $F(\b)$ also satisfies the Kurdyka-\L{}ojasiewicz property.
Accordingly, combining Theorem~\ref{thm:decrease} and Lemma~\ref{lem:critical}, we have the global convergence property of Algorithm~\ref{alg:coord} and of Algorithm~\ref{alg:palm}  as follows.

\begin{theorem}\label{thm:global} Assume that $\Phi$ is a nonzero Bernstein function  on $[0, \infty)$ such that $\Phi(0)=0$ and $\Phi'(\infty)=0$.
Let the sequence $\left\{\b^{(t)}: {t \in \mathbb N} \right\}$ be generated by Algorithm~\ref{alg:coord} or Algorithm~\ref{alg:palm}.
Under the conditions in Theorem~\ref{thm:decrease}, then the following assertions hold.
\begin{enumerate}
\item[\emph{(i)}] The sequence $\left\{\b^{(t)}: {t \in \mathbb N}\right\}$ has finite length,
\begin{equation}\label{eq:gloabl}
\sum^\infty_{t=0}\Big\|\b^{(t+1)}-\b^{(t)} \Big\| <\infty.
\end{equation}
\item[\emph{(ii)}] The sequence $\left\{\b^{(t)}: {t \in \mathbb N} \right\}$ converges to a critical point $\b^*=(b^*_1, \ldots, b^*_p)^T$ of $F$.
\end{enumerate}
\end{theorem}

The  convergence property of Algorithm~\ref{alg:palm} is a direct corollary of the results studied by
\citet{BolteMathProgram}. Notice that in the proof for  Algorithm~\ref{alg:palm}, it is not necessarily required that $P''(b)$
on $(0, \infty)$ exists.
Thus, when we use the MCP or SCAD function in Algorithm~\ref{alg:palm}, the resulting procedure also has the convergence results
in Theorems~\ref{thm:decrease} and \ref{thm:global} because both MCP and SCAD satisfy the Kurdyka-\L{}ojasiewicz property
(see Section~\ref{sec:example}).

It is worth poiniting out that the convergence analysis of  \citet{MazumderSparsenet:11} (see Theorem 4 therein)
has only established the  square summable result $\sum_{t=0}^{\infty}\|\b^{(t+1)}- \b^{(t)} \|^2< + \infty$.  However,
by the square summable result, it can  not be directly  obtained that $\left\{\b^{(t)} \right\}_{t \in \NB}$ is a Cauchy sequence.
The theory of the Kurdyka-\L{}ojasiewicz inequality is an essential tool to obtain this result.

Theorem~\ref{thm:global} says that $\b^*$ is a critical point of $F$.   Let $S = \mathrm{supp}(\b^*)$ and $S^c$ denote the complement of $S$ in $\{1, \ldots, p\}$.   Hence,  when taking Algorithm~1 for regression, we have
\begin{equation} \label{eqn:ktt00}
\0 \in \X^T \X \b^* - \X^T \y + \lambda_n \z,
\end{equation} 
where $\z=(z_1, \ldots, z_p)^T$ and $z_j= \alpha  \Phi'(\alpha |b_j^{*}|) \partial |b_j^*|$. 
Let  $\z_S$ be the sub-vector of $\z$ with entries in $S$ and $\X_S$ is the submatrix of $\X$ with columns indexed by $S$.  Then the Karush-Kuhn-Tucker (KKT)  condition in \eqref{eqn:ktt00} is equivalent to 
\[
\X_S^T \X_S \b_S^* - \X_S^T \y + \lambda_n \z_{S}  = \0
\]
and $\| \X_{S^c}^T \X_{S^c} \b_{S^c}^* - \X_{S^c}^T \y\|_{\max} \leq \lambda_n \alpha \Phi'(0)$. 

When $L = \sum_{i=1}^n \log(1 + \exp(-y_i \b^T \x_i))$  is defined in the classification problem,  we have 
 \[
 \0 \in  \lambda_n \z - \sum_{i=1}^n  \omega_i y_i \x_i  =  \lambda_n \z  -  \X^T \D \y.
 \]
 where $\omega_i = \frac{ \exp(-y_i \b^T \x_i) }{1+  \exp(-y_i \b^T \x_i)} $ and $\D = \diag(\omega_1, \ldots, \omega_n)$.  The current KKT condition is equivalent to 
\[
\lambda_n \z_{S}  -  \X_{S}^T \D \y =\0
\] 
and $\| \X_{S^c}^T \D \y\|_{\max} \leq \lambda_n \alpha \Phi'(0)$. Notice that $\frac{\partial^2 L}{\partial \b \partial \b^T} = \X^T (\I_n-\D) \D \X$.  
The following theorem is a direct corollary of Theorem~1 of  \cite{LvFan:2009,FanLv:2011}. It shows that  $\b^*$ is a strict local minimizer of $F$ under certain regularity conditions.

\begin{theorem} \label{thm:minimizer} Assume that the conditions in Theorem~\ref{thm:global} are satisfied. If $\lambda_{\min}(\X_S^T \X_S) +  \frac{\lambda_n \alpha^2}{\Phi(\alpha)} \Phi''(0) >0$ and $\| \X_{S^c}^T \X_{S^c} \b_{S^c}^* - \X_{S^c}^T \y\|_{\max} < \lambda_n \alpha \Phi'(0)$ in Algorithm~1 or $\lambda_{\min}(\X_S^T (\I_n - \D) \D \X_S) +  \frac{\lambda_n \alpha^2}{\Phi(\alpha)} \Phi''(0) >0$ and $\| \X_{S^c}^T \D \y\|_{\max} < \lambda_n \alpha \Phi'(0)$ in Algorithm~2, then $\b^*$ is a strict local minimizer of $F$. Here $\lambda_{\min}(\A)$ denotes the smallest eigenvalue of the positive semidefinite  matrix $\A$. 
\end{theorem}

Notice that without the condition $\| \X_{S^c}^T \X_{S^c} \b_{S^c}^* - \X_{S^c}^T \y\|_{\max} < \lambda_n \alpha \Phi'(0)$ or $\| \X_{S^c}^T \D \y\|_{\max} < \lambda_n \alpha \Phi'(0)$,
we only can ensure that $\b^*$ is a local minimizer of $F$.

\section{Experimental Analysis} \label{sec:experiment}

In this paper our principal focus has been to explore the theoretical properties of the Bernstein function
in nonconvex sparse modeling. However, we have also developed the coordinate descent
algorithm and the PALM algorithm  for the  supervised learning problems  with the Bernstein penalty.
Thus, it is
interesting to conduct empirical analysis of the estimation algorithms  with different Bernstein penalty functions. We particularly
study the nonconvex  LOG, EXP and LFR functions because they bridge the $\ell_0$-norm and
the $\ell_1$-norm. The MATLAB code will be available at the homepage of the author.

\subsection{Regression Analysis on Simulated Datasets}

First we  conduct  experiments on the regression problem with the  coordinate descent
algorithm based on LOG, EXP and LFR, respectively.
We also implement the lasso, and the $\ell_{1/2}$-norm
and  MCP based
methods~\citep{MazumderSparsenet:11}.
All these methods are also solved by using  coordinate descent.
Moreover, the hyperparameters involved in all the methods are selected via cross validation.

Our empirical analysis is based on a simulated data, which
was used by  \citet{MazumderSparsenet:11}. In particular, we  generate data
from the following model:
\begin{equation*}
 y = \x^T \b + \sigma e
\end{equation*}
where $e \sim N(0, 1)$. 
We choose $\sigma$ such that the Signal-to-Noise Ratio (SNR), which is
$\mathrm{SNR} = \frac{\sqrt{\b^T \Sigma \b}}{\sigma}$,
is a specified value. Following the setting in \citet{MazumderSparsenet:11},  we use $\mathrm{SNR} = 3.0$ in all the experiments.

We generate three different datasets with different $p$ and $n$  to implement the experiments. That is,
\begin{description}
\item[Data {1}:]  $n = 100$, $p = 200$,
$\b_{1}$ has $10$ non-zeros such that $b_{1, 20i{+}1}=1$ for $i=0, 1, \cdots, 9$,
and $\Si_{1} = \{0.7^{|i-j|}\}_{1\leq i,j \leq p}$.
\item[Data {2}:]  $n = 500$, $p = 1000$,
$\b_{2} = (\b_{1}, \cdots, \b_{1})$,
and $\Si_{2} = \mathrm{diag}(\Si_{1}, \cdots, \Si_{1})$ (five blocks).
\item[Data {3}:]  $n = 500$, $p = 2000$,
$\b_{3} = (\b_{1}, \cdots, \b_{1})$,
and $\Si_{3} = \mathrm{diag}(\Si_{1}, \cdots, \Si_{1})$ (ten blocks).
\end{description}

Our experimental analysis is performed on the above training datasets, and the corresponding test datasets
of $m=10000$ samples. Let $\hat{\b}$
denote the solution obtained from each algorithm.  We use a standardized
prediction error (SPE) and a feature selection error (FSE)  as measure metrics.  SPE  is defined as
$\textrm{SPE} = \frac{\sum_{i=1}^{m} (y_i - \x_i^T \hat{\b})^2}{m \sigma^2}$
and  FSE is proportion of coefficients in
$\hat{\b}$ which is incorrectly set to zero or nonzero based on the true $\b$.

Table~\ref{tab:sim_res2} reports the average results over 25 repeats. From
them, we can see that all the methods are competitive
in both prediction accuracy and feature selection accuracy.
But nonconvex penalization outperforms  convex penalization in sparsity.
Although there does not exist a closed-form thresholding operator in the coordinate descent method with
EXP, it is still efficient in computational times.
This agrees with our analysis in Section~\ref{sec:cda}.
We find that the  $\ell_{1/2}$ penalty indeed suffers from the numerically unstable problem during
the computation and prediction. It is worth pointing out that \emph{SparseNet} is based on a calibrated version of MCP~\citep{MazumderSparsenet:11}.
Our experiments show that the performance of the conventional MCP is less effective. Thus, the calibration technique
is necessary for MCP. However, we do not apply any calibration techniques to LOG, EXP and LFR in their implementations.
Thus, the Bernstein penalty function is effective and efficient in sparse modeling.

Comparing further the several nonconvex penalty functions, we see that the performance
of LFR is slightly better than those of the remainders.
Recall that for any fixed $\alpha>0$,
\[
1 {-} \exp( {-} \alpha |b|) \leq \frac{2 \alpha |b|}{\alpha |b| {+}2} \leq  \log\big({\alpha} |b|  {+}1 \big)
\leq  \alpha |b|,\]
with equality only when $b=0$. However, we have seen that related to EXP, LFR has the closed-form thresholding operator.
This would be an important reason why LFR has the best performance. In summary,
the experimental results show that LFR is an especially good choice
for nonconvex penalization in finding spare solutions.

\begin{table*}[!ht]\setlength{\tabcolsep}{1.3pt}
\caption{Results of the coordinate descent algorithms with different penalty functions on the simulated data sets.
\label{tab:sim_res2}}
\begin{center}
\begin{tabular}{l | c  c c  | c  c   c | c c c }
\hline
    &  SPE($\pm$STD) & ``FSE'' &  &  SPE($\pm$STD) & ``FSE'' &  & SPE($\pm$STD) & ``FSE'' &  \\
    \hline
    & \multicolumn{3}{c|}{Data  1}
    & \multicolumn{3}{c|}{Data  2}
    & \multicolumn{3}{c}{Data  3}
\\
 & \multicolumn{3}{c|}{ $p=200$ and $n=100$}
    & \multicolumn{3}{c|}{$p=1000$ and $n=500$}
    & \multicolumn{3}{c}{$p=2000$ and $n=500$}
\\ \hline
LOG  & 1.2307 ($\pm$0.0131) &  0.0146 & & 1.2192($\pm$0.0028) & 0.0121   &  & 3.5200($\pm$1.1410) & 0.0739   & \\
EXP  & 1.1452 ($\pm$0.0074) & {\bf 0.0037}  &  & 1.1407($\pm$0.0013) & 0.0024  &   & 3.6499($\pm$0.1800) & 0.0785   &\\
LFR & {\bf 1.1145} ($\pm$0.0093) & {0.0050}  & & {\bf 1.1205} ($\pm$0.0018) & {\bf 0.0018}   &  & {\bf 3.4109} ($\pm$0.1590) & {\bf 0.0610}   & \\
\hline
$\ell_{1/2}$  & 1.2480 ($\pm$0.0230) & 0.0277 & & 1.2689($\pm$0.0071) & 0.0262   &  & 3.7468($\pm$0.2041) & 0.0896   & \\
MCP      & 1.1195 ($\pm$0.0041) & 0.0051 &  & 1.2736($\pm$0.0509) & 0.0430   &   & 3.6853($\pm$0.1580) & 0.0821   &\\
\hline
 Lasso  &  1.6678 ($\pm$0.0654) & .01555  & & 1.6588($\pm$0.0184) & 0.1520   &   & 4.0433($\pm$0.1607) & 0.1470   &\\
\hline
\end{tabular}
\end{center}
\end{table*}

\subsection{Classification Analysis on Real Datasets}

We now conduct empirical analysis of the classification problem with the PALM algorithm based on the Bernstein penalty functions.
More specifically, our experiments are performed on four real datasets. The \emph{heart} data (270 samples and 13 features),  the
\emph{Australian} data (690 samples and 14 features), and
the \emph{German number} data (1000 samples and 24 features)   come from Statlog.
The  \emph{splice} dataset  (1000 samples and 60 features) is from Delve.
The datasets are  used for the binary classification problem.

For comparison, we also implement the conventional SVM and  the penalized logistic regression with the $\ell_1$ penalty.
We use the $70\%$ of the data for training and the rest $30\%$ for testing. Table~\ref{tab:class} reports the average results over 30 repeats.
We see that the methods based on the Bernstein penalty functions slightly outperform the two convex methods.
Figure~\ref{fig:converg} illustrates the convergence results of the PALM procedures with EXP, LOG, and LFR, respectively.
As we see, the PALM method for the nonconvex penalization problem admits the convergence property.

It is worth pointing out that in the PALM algorithm the input samples  are not necessarily standardized such that $\sum_{i=1}^n x_{ij}=0$
and $\sum_{i=1}^n x^2_{ij}=1$. However, the coordinate descent methods studied by \citet{MazumderSparsenet:11}
and \citet{BrehenyAAS:2010} typically require such standardization.

\begin{table}[!ht] 
\caption{Classification accuracies}
\label{tab:class}
\vspace{0.1in}
\begin{center}
\begin{tabular}{l| c|c|c|c}
\hline
	{}      & ~~~~\thead {Heart\\ $n=270,p=13$}~~~ &~~\thead{Australian\\$n=690,p=14$}~~&~~\thead{German number\\$n=1000,p=24$}~~&~~\thead{Splice\\$n=3175, p=60$}~~\\
\hline
    {LOG}   &  $ 0.8463$ ($\pm 0.0340$)   & $ 0.8589$  ($\pm 0.0196 $) &  $ 0.7603$ ($\pm 0.0236$) & $0.7997$ ($\pm 0.0150$) \\
    {EXP} &   ${\bf 0.8639}$ ($\pm 0.0296$)     &  ${\bf 0.8638}$ ($\pm 0.0221$) &  $ 0.7602$  ($\pm 0.0237$) & ${\bf 0.8037} $ ($\pm 0.0241$)  \\
    {LFR} &   $ 0.8620$ ($\pm 0.0297$)     & $ 0.8638 $ ($\pm 0.0222$) &  $ {\bf 0.7622}$ ($\pm 0.0248 $) & $0.7999$ ($\pm 0.0201$)  \\
\hline
    {$\ell_1$-norm} &   $ 0.8231 $ ($\pm 0.0212$)    &  $ 0.8595$ ($\pm 0.0194 $)  &  $ 0.7356$ ($\pm 0.0230 $) & $ 0.7875$ ($\pm 0.0077$) \\
\hline
    {SVM} &   $0.8417$ ($\pm 0.0455$)      &  $ 0.8522 $ ($\pm 0.0223 $)&  $ 0.7600$ ($\pm 0.0226$) & $0.7890 $ ($\pm 0.0235$) \\
\hline
\end{tabular}
\end{center}
\end{table}

\begin{figure}[!ht]
\centering
\subfigure[``Heart" data]{\includegraphics[width=75mm,height=60mm]{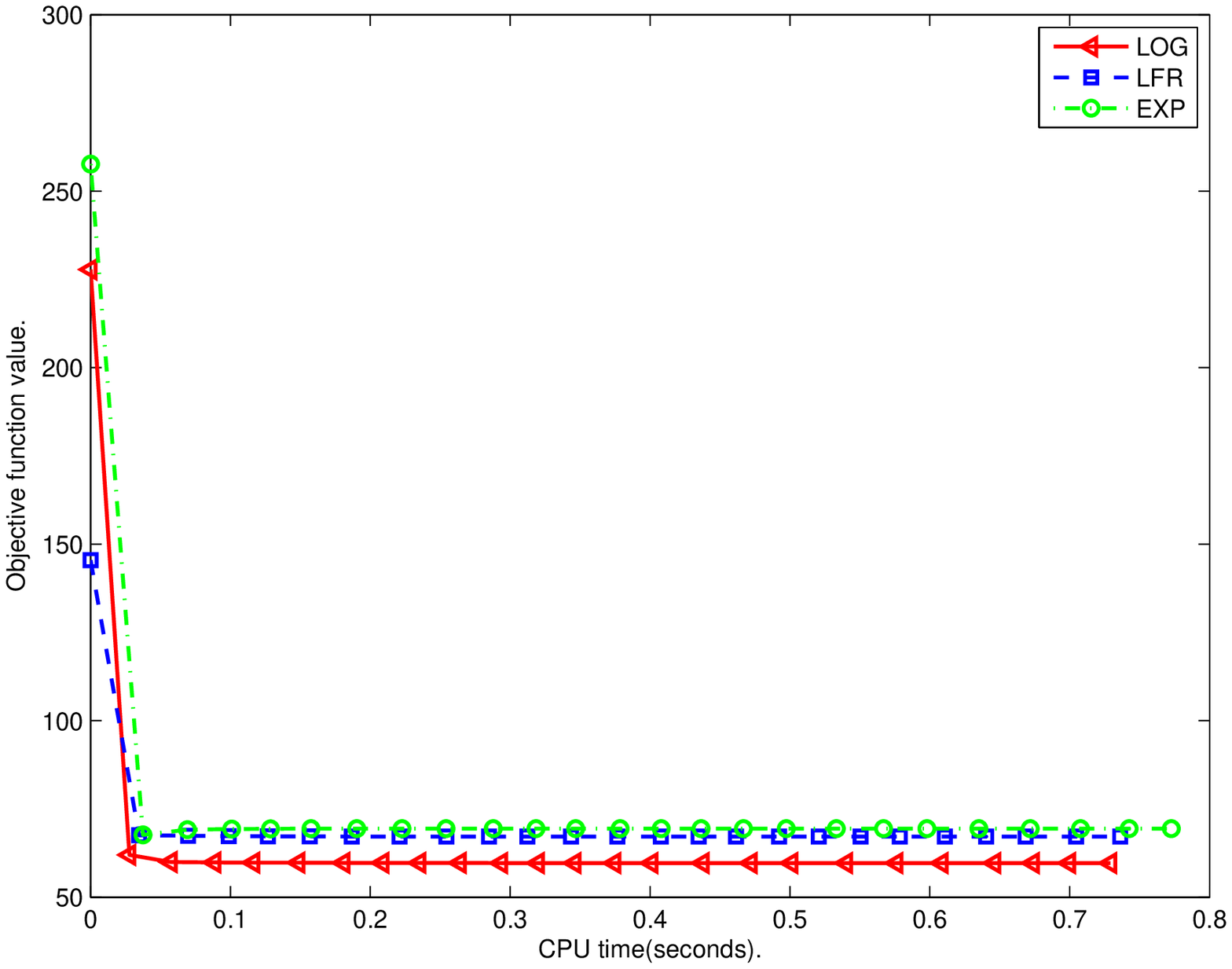}}  \subfigure[``Australian" data]{\includegraphics[width=75mm,height=60mm]{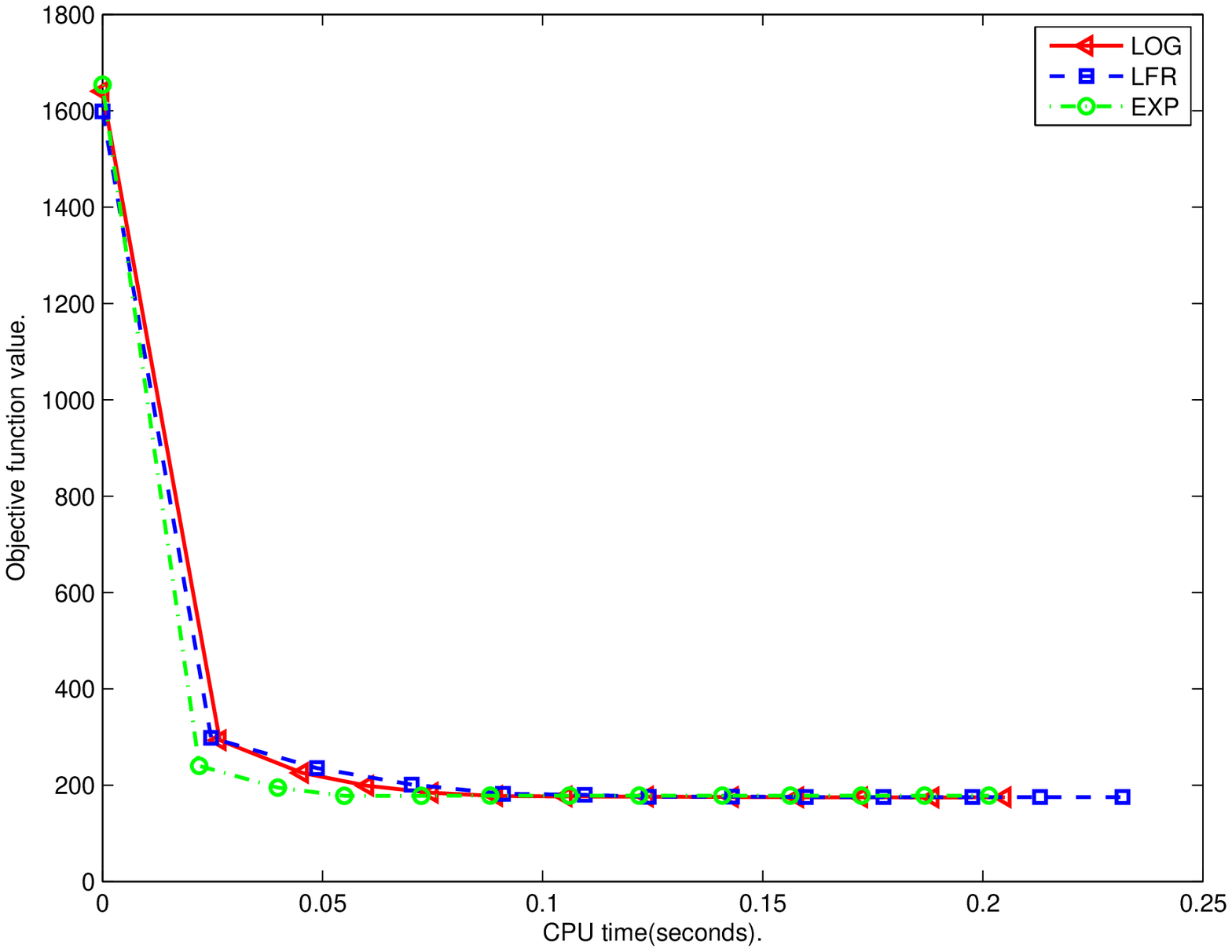}} \\
\subfigure[``German number" data]{\includegraphics[width=75mm,height=60mm]{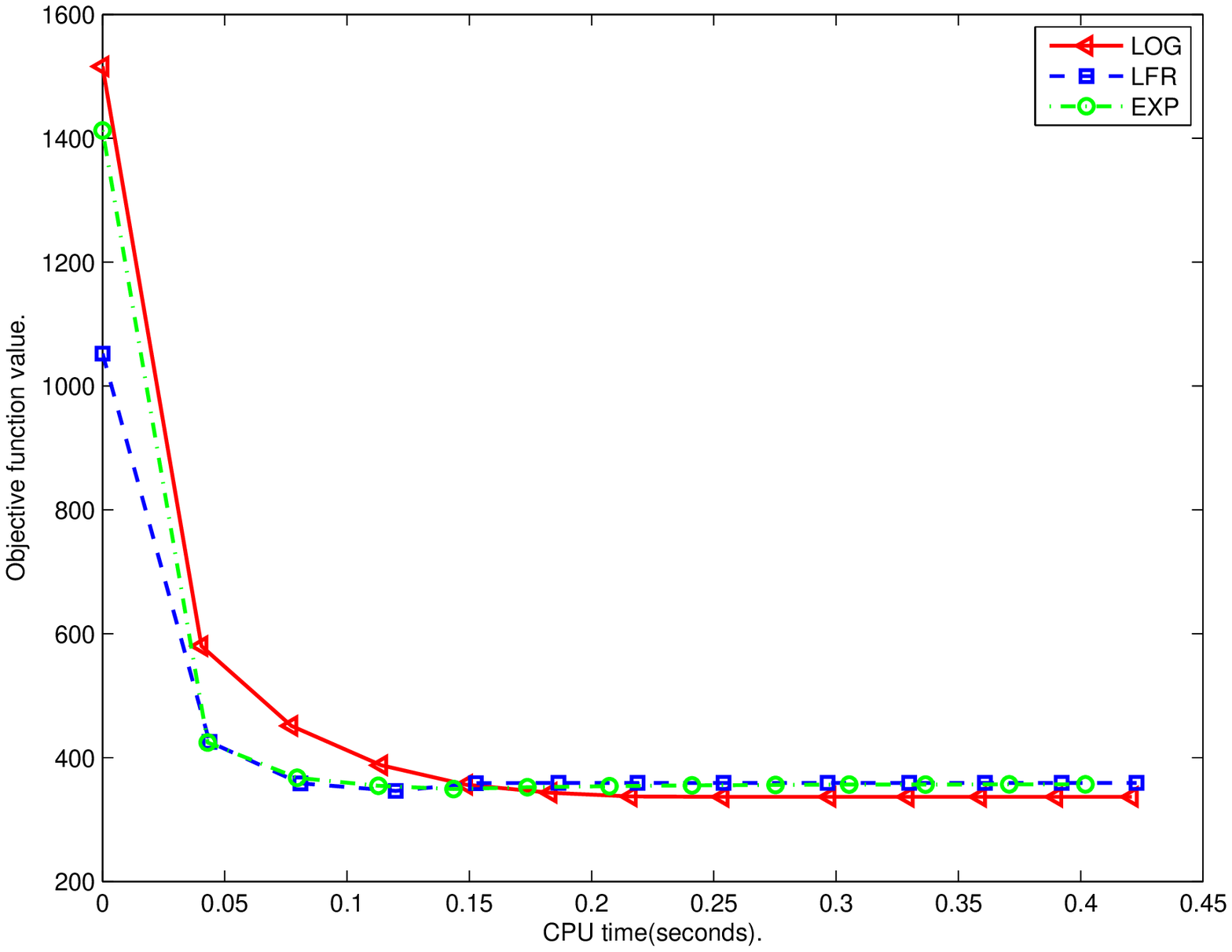}}
\subfigure[``Splice" data]{\includegraphics[width=75mm,height=60mm]{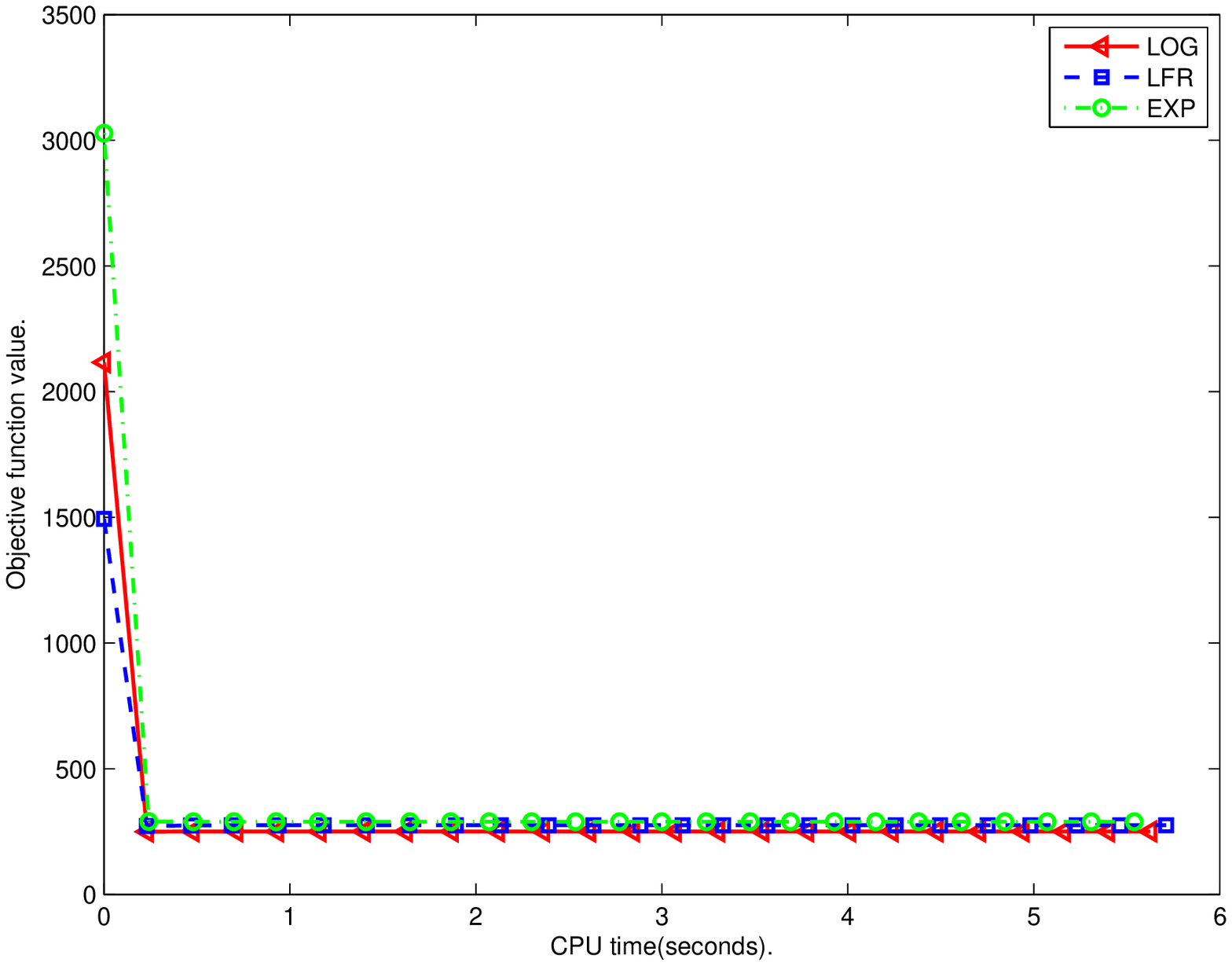}} \\
\caption{Convergence results of the PALM procedure.}
\label{fig:converg}
\end{figure}

Finally, we apply the classification method based on the PALM algorithm on the microarray gene expression data of  leukemia
patients~\citep{GolubT:1999}.
This  data involves 7129 genes for 72 patients. Following the treatment in \cite{BrehenyAAS:2010}, we use 38 patients for training and the other 34 for testing. 
We implement the PALM algorithm with LOG, EXP, LFR, and MCP, respectively. For comparison,
we also implement the   method of \citet{BrehenyAAS:2010},
which ia based on the second order Taylor approximation at the current estimate value of
the regression vector.
We report the misclassification error evaluated on the leukemia dataset. We point out that there can be $33/34$ accuracy adopted by the PALM with the either of LOG, EXP, FLR, and MCP,  while $31/34$ used by the  method of \citet{BrehenyAAS:2010}.

\section{Conclusion} 
\label{sec:conclusion}

In this paper we have exploited Bernstein functions in the definition of nonconvex penalty functions.
To the best of our knowledge, it is the first time that we apply theory of Bernstein functions
to systematically study nonconvex penalization problems.
We have illustrated  the KEP, LOG, EXP and LFR functions, which have wide applications in many scenarios but sparse modeling.
We have conducted empirical analysis with  LOG, EXP and LFR, which
shows they are good choices.

The Bernstein function has attractive ability in sparsity modeling. Geometrically, the Bernstein function holds  the property of regular variation~\citep{FellerBook:1971}. In other words, the Bernstein function
bridges the $\ell_{q}$-norm ($0\leq q <1$) and the $\ell_1$-norm.  Computationally, the resulting
estimation problems can be efficiently solved by using coordinate descent algorithms. The Bernstein function enjoys the Kurdyka-\L{}ojasiewicz property~\citep{Lojasiewicz,Kurdyka,BolteDaniilidisLewis}, which makes the coordinate descent procedure have global convergence properties.   Theoretically,
the Bernstein function admits the oracle properties (more details given in the appendix) 
and can result in an unbiased and continuous sparse estimator.


\appendix

\section{Several Important Results on Bernstein functions}

In this section we present several lemmas that are useful for Bernstein functions.

\begin{lemma} \label{lem:lapexp}
Let $\Phi(s)$ be a nonzero Bernstein function of $s$ on $(0, \infty)$. Assume $\liml_{s \to 0}\Phi(s)=0$ and $\liml_{s\to \infty} \frac{\Phi(s)}{s} =0$.
Then
\begin{enumerate}
\item[\emph{(a)}] $\liml_{s\rightarrow +\infty} \Phi^{(k)}(s) =0$ and  $\liml_{s\rightarrow 0+} s^k \Phi^{(k)}(s) =0$ for any $k\in \NB$. Additionally,
if $\liml_{s \to \infty} \Phi(s) < \infty$, then $\liml_{s\rightarrow \infty} s^k \Phi^{(k)}(s) =0$ for $k\in \NB$.
\item[\emph{(b)}] If  $\liml_{s \rightarrow \infty} s \Phi'(s)$ exists (possibly infinite), then $\liml_{s \rightarrow \infty} \frac{(-1)^{k{-}1}}{(k{-}1)!} s^k \Phi^{(k)}(s)$ for all $k \in \NB$
exist and are identical. In fact, if $\Phi'(0) = \liml_{s \to 0+} \Phi'(s) =1$, then  $\liml_{s \rightarrow \infty} s \Phi'(s) = \liml_{u\to 0+} \frac{F(u)}{u}$ where
$F(u)$ is the probability distribution which concentrated on $(0, \infty)$ and  whose Laplace transform is $\Phi'(s)$.
\end{enumerate}
\end{lemma}

\begin{proof}
First,
it follows from the L\'{e}vy-Khintchine representation
that
\[
\Phi(s) =  \int_{0}^{\infty} \big[1- e^{- s u}  \big ] \nu(d u)
\]
due to $\Phi(0)=0$ and $\liml_{s \to \infty} \frac{\Phi(s)}{s}=0$.
Thus, we have
\[
\Phi^{(k)}(s) = (-1)^{k-1} \int_{0}^{\infty}  e^{- s u} u^{k} \nu(d u).
\]
When $s\geq k$ for any $k\in \NB$, it is easily verified  that $e^{-su} u^k \leq \frac{u^k}{1+u^k}$ for $u>0$. Notice that
\[
\int_{0}^{\infty}{\min(u^k, 1) \nu(d u)} \leq \int_{0}^{\infty}{\min(u, 1) \nu(d u)} < \infty
\]
and
\[
\frac{u^k}{1+u^k} \leq \min(u^k, 1) \leq \frac{2 u^k}{1+u^k}, \quad u\geq 0.
\]
This implies that $\int_{0}^{\infty}{\min(u^k, 1) \nu(d u)} < \infty$ is equivalent to that  $\int_{0}^{\infty}{ \frac{u^k}{1+u^k} \nu(d u)} < \infty$.
As a result, we have that when $s\geq k$,
\[
 \int_{0}^{\infty}  e^{- s u} u^k \nu(d u) = \int_{0}^{\infty}  e^{- s u} u^k \nu(d u) \leq \int_{0}^{\infty}{ \frac{u^k}{1+u^k} \nu(d u)} < \infty.
\]
Thus,
\[
\lim_{s \to \infty} \Phi^{(k)}(s) = (-1)^{k-1} \lim_{s \to \infty} \int_{0}^{\infty}  e^{- s u} u^k \nu(d u)
= (-1)^{k-1} \int_{0}^{\infty} \lim_{s \to \infty}   e^{- s u} u^k  \nu(d u) =0.
\]

Additionally, since $e^{-s u} (su)^k \leq k^k e^{-k}$ for $s \geq 0$ and $u \geq 0$, we have
\begin{align}
\int_{0}^{\infty}  e^{- s u} (su)^k \nu(d u) &= \int_{0}^{1}  e^{- s u} (su)^k \nu(d u) +
 \int_{1}^{\infty}  e^{- s u} (su)^k \nu(d u) \nonumber \\
 & \leq\int_{0}^{1}  e^{- s u} (su)^k \nu(d u) + \int_{1}^{\infty}{ k^{k} e^{-k} \nu(d u)}. \label{eqn:a1}
\end{align}
Hence,
for any $s\leq 1$,
\[
\int_{0}^{\infty}  e^{- s u} (su)^k \nu(d u) \leq\int_{0}^{1}  u \nu(d u)
+ \int_{1}^{\infty}{ k^{k} e^{-k} \nu(d u)} \leq  \max(1,  k^{k} e^{-k}) \int_{0}^{\infty} {\min(1,u)  \nu(d u)} < \infty.
\]
As a result, we obtain
\[
\lim_{s \to 0} s^k \Phi^{(k)}(s) = (-1)^{k-1} \lim_{s \to 0} \int_{0}^{\infty}  e^{- s u} (su)^k \nu(d u)
= (-1)^{k-1} \int_{0}^{\infty} \lim_{s \to 0}  e^{- s u} (su)^k \nu(d u)=0.
\]

Furthermore,
$\liml_{s\to \infty}\Phi(s)=M_0<\infty$ implies that $\int_{0}^{\infty} {\nu(d u) } < \infty$, so we always have
\[
\int_{0}^{\infty}  e^{- s u} (su)^k \nu(d u) \leq  k^k e^{-k} \int_{0}^{\infty} {\nu(d u) } < \infty,
\]
which leads us to $\liml_{s\to \infty} s^k \Phi^{(k)}(s)=0$ for any $k\in \NB$.

We now prove Part (b). Consider  that
\[
\frac{(-1)^{k-1} s^k \Phi^{(k)}(s)}{(k{-}1)!}
=   \int_{0}^{\infty} \frac{s^{k}}{(k{-}1)!} e^{- s u} u^{k{-}1}  u \nu(d u)
\]
and that $ \frac{s^{k}}{(k{-}1)!} e^{- s u} u^{k{-}1}$ is the p.d.f.\ of gamma random variable $u$ with
shape parameter $k$ and scale parameter $1/s$. Such a gamma random variable converges to the Dirac Delta measure $\delta_0(u)$ in  distribution
as $s \to +\infty$.
For a fixed $u>0$,
$\frac{s^{k}}{(k{-}1)!} e^{- s u} u^{k{-}1}$ is monotone w.r.t.\ sufficiently large $s$.
Accordingly, using monotone convergence,
we have
 \begin{align*}
\lim_{s \to \infty} \frac{(-1)^{k-1} s^k \Phi^{(k)}(s)}{(k{-}1)!}
& =  0 \nu(\{0\})  + \lim_{s \to \infty}  \int_{0+}^{\infty} \frac{s^{k}}{(k{-}1)!} e^{- s u} u^{k{-}1}  u \nu(d u) \\
& = 0 \nu(\{0\})=  \int_{0}^{\infty} {\delta_0(u) u \nu(du)} = \liml_{s \rightarrow \infty} s \Phi'(s).
\end{align*}

When $\Phi'(0) = \liml_{s \to 0+} \Phi'(s) =1$, it is a well-known result that $\Phi'(s)$ is the Laplace transform of some probability distribution
(say, $F(u)$). That is,
\[
\Phi'(s) = \int_{0}^{\infty}{\exp(- s u) d F(u)}= \int_{0}^{\infty}{ s \exp(- s u) F(u) d u}.
\]
Here we use  $F(0)=0$ because $F$ is concentrated on $(0, \infty)$.
Recall that $s^2 u \exp(-su) \to $ $\delta_0(u)$ in distribution
as $s \to +\infty$. We  thus have
\[
\liml_{s \to \infty } s \Phi'(s) = \lim_{u \to 0+} \frac{F(u)}{u}.
\]
This result can be also obtained from Tauberian Theorem~\citep{Widder:1946}.
Furthermore, if $F(u)$ is the probability distribution of some continuous nonnegative random variable $U$, we have $\liml_{s \to \infty } s \Phi'(s) = F'(0+)$. If $U$ is discrete, we see two cases. In the first case, $\Pr(U=0)>0$. This implies that $F(u)\geq \Pr(U=0)>0$ for any $u>0$. Thus, we have
$\liml_{s \to \infty } s \Phi'(s) = \liml_{u \to 0+} \frac{F(u)}{u} = \infty$.
In the second case, $\Pr(U=0)=0$.
Then there exists a small positive number $\delta$ such that $F(u)=0$ for $u>\delta$. As a result, we obtain that $\liml_{s \to \infty } s \Phi'(s) = \liml_{u \to 0+} \frac{F(u)}{u} = 0$.
\end{proof}

\begin{lemma} \label{lem:01} Let $\Phi$ be a nonzero Bernstein function  on $[0, \infty)$ such that $\liml_{s \to \infty} s \Phi'(s)$ is finite. Then we have $\liml_{\s \to \infty} \frac{\Phi(s)}{\log(1+s)} = \liml_{s \to \infty} s \Phi'(s) < \infty$. Furthermore, we have
\[
\lim_{s \to \infty} \frac{s \Phi'(s)}{\Phi(s)} =0.
\]
\end{lemma}
\begin{proof}
It follows from the condition $\liml_{s \to \infty} s \Phi'(s) < \infty$ that
$\liml_{s \to \infty}\frac{\Phi(s)}{\log(1+s)}=
\liml_{s \to \infty}\frac{\Phi(s)}{\log(s)}=\liml_{s \to \infty} s \Phi'(s) < \infty$. Thus, when $\liml_{s \to \infty} \Phi(s) = \infty$, we have
$ \lim_{s \to \infty} \frac{s \Phi'(s)}{\Phi(s)} =0$.
Otherwise $\liml_{s \to \infty} \Phi(s) =M \in (0, \infty)$,  we always have that
$\liml_{a \to \infty}\frac{\Phi(s)}{\log(1+s)}=\liml_{s \to \infty}\frac{\Phi(s)}{\log(s)}= \liml_{s \to \infty} s \Phi'(s) = 0$.
Thus, we  have $\liml_{s \to \infty} \frac{s \Phi'(s)}{\Phi(s)}=0$ in any case.
\end{proof}

\begin{lemma} \label{lem:03} Let $\Phi$ be a nonzero Bernstein function  on $[0, \infty)$. Assume  $\Phi(0)=0$, $\Phi'(0)=1$, and $\Phi'(\infty)=0$.
Then $\liml_{s \to \infty} \frac{s \Phi'(s)}{\Phi(s)} \in [0, 1]$.
\end{lemma}
\begin{proof}  Lemma~\ref{lem:lapexp}  shows that $\liml_{s \to \infty} \frac{{s \Phi'(s)}}{\Phi(s)}= 0$
whenever $\liml_{s \to \infty} \Phi(s)< \infty$.  If $\liml_{s \to \infty} \Phi(s)= \infty$, we take $\Psi(s) \triangleq \log(1+\Phi(s))$,
which is also Bernstein and holds the conditions $\Psi(0)=0$,  $\Psi'(0)=1$ and $\Psi'(\infty)=0$.  In this case,
$\liml_{s \to \infty} \frac{{s \Phi'(s)}}{\Phi(s)}=  \liml_{s \to \infty} s \Psi'(s)$
due to  $\Psi'(s) = \frac{\Phi'(s)}{1+ \Phi(s)}$.
Thus, Lemma~~\ref{lem:lapexp}-(b)  directly applies the Bernstein function $\Psi(s)$. Thus, $\liml_{s \to \infty} \frac{{s \Phi'(s)}}{\Phi(s)}$ exists.

Now consider that $s\Phi'(s)-\Phi(s)$ is a decreasing function on $(0, \infty)$
because its first-order derivative is non-positive; i.e.,  $s \Phi{''}(s)\leq 0$. As a result,
we have $0\leq\frac{s \Phi'(s)}{\Phi(s)} \leq 1$. Subsequently, $\gamma=\liml_{s \to \infty}
\frac{s \Phi'(s)}{\Phi(s)} \in [0, 1]$.
\end{proof}

\section{The Proofs}

In this section we present the proofs of the results given in the paper. 

\subsection{The Proof of Theorem~\ref{thm:lp2}}
\label{app:bb}

\begin{proof}
It is directly verified that
\[
\lim_{\alpha \to 0} \frac{\Phi(\alpha |b|)}{\Phi(\alpha)} = \liml_{\alpha \to 0} \frac{|b| \Phi'(\alpha |b|)}{\Phi'(\alpha)}
=\frac{|b| \Phi'(0)}{\Phi'(0)} =|b|
\]
due  to $\Phi'(0) = 1 \in (0, \infty)$. Clearly, we have that $\liml_{\alpha \to +\infty} \frac{\Phi(\alpha s)}{\Phi(\alpha)} = 0$ when $s=0$ and
that $\liml_{\alpha \to +\infty} \frac{\Phi(\alpha s)}{\Phi(\alpha)} = 1$ when $s=1$.

Lemma~\ref{lem:03} shows that $\gamma=\liml_{s \to \infty}
\frac{s \Phi'(s)}{\Phi(s)} \in [0, 1]$. When $\liml_{s \to \infty} \frac{\Phi(s)}{\log(1+s)} < \infty$, Lemma~\ref{lem:01} implies that $\gamma=0$.
According to Theorem~1 in Chapter VIII.9 of \citet{FellerBook:1971},
we have the second part of the theorem.
\end{proof}

\subsection{The Proof of Proposition~\ref{pro:33}}

\begin{proof} Let $\omega =\frac{1}{1-\rho}$. For $-\infty<\rho\leq 1$, we have $\omega \ (0, \infty]$. We now write $\Phi'_{\rho}(s)$ for a fixed $s>0$ as  $1/g(\omega)$ where
\[
g(\omega) = {(1 + \frac{s}{\omega})^{\omega}}.
\]
It is a well-known result that for a fixed $s>0$ $g(\omega)$ is increasing in $\omega$ on $(0, \infty)$. Moreover, $\liml_{\omega \to \infty} g(\omega) = \exp(s)$.  Accordingly, $\Phi'_{\rho}(s)$  is decreasing in $\rho$ on $(-\infty, 1]$. Moreover, we obtain
\[
\Phi_{\rho_1} (s) = \int_{0}^{s}{\Phi'_{\rho_1}(t) d t } \geq \int_{0}^{s}{\Phi'_{\rho_2}(t) d t } = \Phi_{\rho_2} (s)
\]
whenever $\rho_1\leq \rho_2 \leq 1$.

The proof of Part-(b) is immediately. We here omit the details.
\end{proof}

\subsection{The Proof of Theorem~\ref{thm:sparsty}}

\begin{proof}
The first-order derivative of (\ref{eqn:general}) w.r.t.\ $b$ is
\[
\sgn(b)\big(|b| + {\lambda} \Phi'(|b|) \big) - z.
\]
Let $g(|b|)= |b| + {\lambda} \Phi'(|b|)$. It is clear that if
$|z|< \min_{b\neq 0}\{g(|b| \}$, the resulting estimator is 0; namely, $\hat{b}=0$.
We now check the minimum value
of $g(s)=s + {\lambda} \Phi'(s)$ for $s\geq 0$.

Taking the first-order derivative of $g(s)$ w.r.t.\ $s$, we have
\[
g'(s) = 1 + {\lambda} \Phi{''}(s).
\]
Notice that  $\Phi{''}(s)$ is non-positive and increasing in $s$. As a result, we have
\[
g'(s) \geq 1 + {\lambda} \Phi{''}(0).
\]
Thus, if $\lambda \leq -\frac{1}{\Phi{''}(0)}$, $g(s)$  attains its minimum value ${\lambda}\Phi'(0)$ at $s^{*}=0$.
Otherwise,  $g(s)$ attains its minimum value when $s^{*}$ is the solution of
 $1 + {\lambda} \Phi{''}(s)=0$.

First, we consider the case that $\lambda \leq -\frac{1}{\Phi{''}(0)}$. In this case, the resulting estimator is 0 when $|z|\leq {\lambda} \Phi{'}(0)$. If $z> {\lambda}\Phi{'}(0)$, then the resulting estimator should be a positive root of the equation
$b + {\lambda} \Phi{'}(b) - z = 0$ in $b$. Letting $h(b)=b + {\lambda} \Phi{'}(b) - z$,
we study the roots of $h(b)=0$.
Notice that
$h(z) =  {\lambda} \Phi{'}(z) >0$ and
$h(0)=  {\lambda} \Phi{'}(0) - z <0$. In this case, moreover, we have that $h(b)$ is increasing on $[0, \infty)$.
This implies that  $h(b)=0$
has one and only one positive root.
Furthermore, the resulting estimator $0<\hat{b}<z$ when $z> {\lambda}\Phi{'}(0)$.
Similarly, we can obtain that $z<\hat{b}<0$ when $z<- {\lambda} \Phi{'}(0)$.
As stated in \citet{Fan01}, a sufficient and necessary condition for ``continuity" is
the minimum of
$|b|+ {\lambda} \Phi'(|b|)$ is attained at $0$. This implies that that the resulting estimator is continuous.

Next, we  prove the case that $\lambda > -\frac{1}{\Phi{''}(0)}$. In this case, $g(s)$ attains its  minimum value
$g(s^*)= s^* + {\lambda} \Phi'(s^*)$ when $s^*$ is the solution of equation
$1 + {\lambda} \Phi{''}(s)=0$. Notice that  $\Phi{''}(s)$ is non-positive and increasing in $s$.
Thus, the solution $s^*$ exists and
is unique. Moreover, since $\Phi{''}(s^*) = - \frac{1}{\lambda} > \Phi{''}(0)$, we have $s^*> 0$.
In this case, the resulting estimator is 0 when $|z|\leq s^* + {\lambda} \Phi'(s^*)$.
We just make attention on the case that $|z|> s^* + {\lambda} \Phi'(s^*)$.
Subsequently, the resulting estimator is $\hat{b}=\sgn(z)\kappa(|z|)$
where $\kappa(|z|)$ should be a positive root of equation $b + {\lambda} \Phi{'}(b) - |z| = 0$.
We now need to prove that
$\kappa(|z|)$  exists and is unique on $(s^*, |z|)$.
We have that $h(b)=b + {\lambda} \Phi{'}(b) - |z| $  is a convex
function of $b$ on $[0, \infty)$ due to  $h{''}(b)= {\lambda} \Phi{'''}(b) \geq 0$.
This implies that $h(b)$ is increasing on $[s^*, \infty)$ and decreasing on $(0, s^*)$.
Thus, the equation $h(b)=0$ has at most two positive roots, which are on $(0, s^*)$ or $[s^*, \infty)$.
Since $h(s^*)= s^* + {\lambda} \Phi'(s^*)-|z|<0$
and $h(|z|)= {\lambda} \Phi'(|z|) \geq 0$, the equation $h(b)=0$ has one unique root on $(s^*, |z|)$.
Thus, $\kappa(|z|)$ exists and is unique on $(s^*, |z|)$. It is worth pointing
out that if the equation $h(b)=0$ has a root on $(0, s^*)$, the objective function $J_1(b)$ attains its maximum value at this root.
Thus, we can exclude this root.
\end{proof}

\subsection{The Proof of Proposition~\ref{thm:nest}}

Observe that $1 = \Phi'(0)=\int_{0}^{\infty}{u \nu(d u)}$ and $\Phi(\alpha)= \int_{0}^{\infty}{(1-\exp(-\alpha u)) \nu(d u)}$.
Since $\alpha u>1-\exp(-\alpha u)$ for $u>0$, we obtain $\Phi(\alpha)<\alpha$.
Additionally, $\Big[\frac{\alpha}{\Phi(\alpha)} \Big]' =\frac{\Phi(\alpha) - \alpha \Phi'(\alpha)}{\Phi^2(\alpha)} \geq 0$
due to $[\Phi(\alpha) - \alpha \Phi'(\alpha)]' = -\Phi{''}(\alpha) \geq 0$. Also, $\Big[\frac{1}{\Phi(\alpha)} \Big]' \leq 0$.
We thus obtain that $\frac{\alpha}{\Phi(\alpha)}$ is increasing, while $\frac{1}{\Phi(\alpha)}$ is decreasing. Furthermore,
we can see that $\liml_{\alpha \to 0+} \frac{\alpha }{\Phi(\alpha)}=\liml_{\alpha \to 0+} \frac{1 }{\Phi'(\alpha)}=1$ and $\liml_{\alpha \to \infty} \frac{\alpha }{\Phi(\alpha)} = \liml_{\alpha \to \infty} \frac{1}{\Phi'(\alpha)}= \infty$.

\subsection{The Proof of Theorem~\ref{thm:0-1}}

\begin{proof}
First, it is easily obtained that
$\lim\limits_{\alpha \to 0} \frac{\alpha} {\Phi(\alpha)} = \frac{1}{\Phi'(0)}$
and $\lim\limits_{\alpha \to 0}  \frac{\Phi(\alpha)}{\alpha^2} = \infty$. This implies that in the limiting case the condition $\eta \leq -\frac{\Phi(\alpha)}{\alpha^2 \Phi{''}(0)}$ is always met (i.e., Case (i) in Theorem~\ref{thm:sparsty}).
Moreover, $|z| > \frac{\eta \alpha}{\Phi(\alpha)} \Phi{'}(0)$ degenerates to $|z|> \eta$.
In addition, we have
\[
\lim_{\alpha \to 0} \frac{\alpha \Phi'(\alpha b)}{\Phi(\alpha)} = \lim_{\alpha \to 0} \frac{\Phi'(\alpha b) + \alpha b \Phi{''}(\alpha b)}{\Phi'(\alpha)} =1.
\]
This implies that $\kappa(|z|)$ converges to the nonnegative solution of equation of the form
\[
b + \eta - |z|=0.
\]
That is,  $\kappa(|z|)=|z| - \eta$ when $|z|>\eta$.

Second, it is easily obtained that
$\liml_{\alpha \to \infty} \frac{\alpha} {\Phi(\alpha)} = \infty$
and $\liml_{\alpha \to \infty}  \frac{\Phi(\alpha)}{\alpha^2} = 0$.
This implies that in the limiting case the condition $\eta > -\frac{\Phi(\alpha)}{\alpha^2 \Phi{''}(0)}$ is always held.

Recall that $s^*>0$ is the unique root of $1+ {\lambda} \Phi{''}(s)=0$ and $\Phi{''}(s)$ is monotone increasing, so we can express
$s^*$ as $s^*= \frac{1}{\alpha} (\Phi{''})^{-1}(- \Phi(\alpha)/(\eta \alpha^2))$. Since $\liml_{\alpha \to \infty} \Phi(\alpha)/(\eta \alpha^2) =0$,
we can deduce that  $\liml_{\alpha \to \infty} (\Phi{''})^{-1}(- \Phi(\alpha)/(\eta \alpha^2))=\infty$. Subsequently,
\[
\lim_{\alpha \to \infty} s^{*} =\lim_{\alpha \to \infty}  \frac{1}{\alpha} (\Phi{''})^{-1}(- \Phi(\alpha)/(\eta \alpha^2))= \lim_{\alpha \to \infty}
\big[(\Phi{''})^{-1}(- \Phi(\alpha)/(\eta \alpha^2))\big]' \leq |z|.
\]
Additionally,
\begin{align*}
\lim_{\alpha \to \infty} \frac{\eta \alpha}{\Phi(\alpha)} \Phi'[(\Phi{''})^{-1}(- \Phi(\alpha)/(\eta \alpha^2))]
&=  \lim_{\alpha \to \infty}
- \frac{[\Phi(\alpha)/\alpha^2]}{\frac{\Phi'(\alpha)}{\alpha} - \frac{\Phi(\alpha)}{\alpha^2} }  \big[(\Phi{''})^{-1}(- \Phi(\alpha)/(\eta \alpha^2))\big]' \\
&= \lim_{\alpha \to \infty}
\big[(\Phi{''})^{-1}(- \Phi(\alpha)/(\eta \alpha^2))\big]' = \lim_{\alpha \to \infty} s^{*}.
\end{align*}
Assume $ \liml_{\alpha \to \infty} s^{*} =c \in (0, |z|]$. Then for sufficiently large $\alpha$, we have
$(\Phi{''})^{-1}(- \Phi(\alpha)/(\eta \alpha^2)) \simeq \alpha$; that is,
\[
\Phi(\alpha) \simeq  - \eta \alpha^2 \Phi{''}(\alpha).
\]
However, if $\liml_{\alpha \to \infty} \Phi(\alpha) <\infty$ then
$-\liml_{\alpha \to \infty} \alpha^2 \Phi{''}(\alpha) =  \liml_{\alpha \to \infty}  \frac{\Phi(\alpha)}{\log(\alpha)}=0$; while $\liml_{\alpha \to \infty} \Phi(\alpha) = \infty$ then
$-\liml_{\alpha \to \infty} \alpha^2 \Phi{''}(\alpha) =\liml_{\alpha \to \infty} \alpha \Phi{'}(\alpha) = \liml_{\alpha \to \infty}  \frac{\Phi(\alpha)}{\log(\alpha)} < \infty$.
This makes the contradiction due to the assumption
$\liml_{\alpha \to \infty} s^{*} =c \in (0, |z|]$. Thus, we have  $\liml_{\alpha \to \infty} s^{*} =0$. Hence,
\[
\lim_{\alpha \to \infty} s^{*} + \frac{\eta \alpha}{\Phi(\alpha)} \Phi'(\alpha s^*) =0.
\]
Finally, we have
\[
\lim_{\alpha \to \infty} \kappa(b) - \frac{\eta \alpha}{\Phi(\alpha)} \Phi'(\alpha \kappa(b)) = |z|,
\]
which implies $\liml_{\alpha \to \infty} \kappa(|z|) = |z|$. The second part now follows.
\end{proof}

\subsection{The proof of Theorem~\ref{thm:decrease}}

\begin{proof} The case for Algorithm~\ref{alg:palm} has been proved  by \citet{BolteMathProgram}. We now only consider the case for
Algorithm~\ref{alg:coord}. Without loss of generality, we assume that $\alpha=1$. The proof is obtained by making slight changes to that for Theorem 4 of \citet{MazumderSparsenet:11}.
Specifically,  let
\[
g(u) \triangleq F(b_1, \ldots, b_{i-1}, u, b_{i+1}, \ldots, b_p) =  L(b_1, \ldots, b_{i-1}, u, b_{i+1}, \ldots, b_p)+ \lambda_n \Phi(|u|).
\]
Then the limiting-subdifferential of $g$ at $u$ is given as
\[
\partial g(u)= \nabla_i L(b_1, \ldots, b_{i-1}, u, b_{i+1}, \ldots, b_p) + \lambda_n \Phi'(|u|) \partial |u|.
\]
Using the strong convexity of $L$  and  Taylor's series expansion of $\Phi(|z|)$ w.r.t.\ $|z|$, we have
\begin{align*}
g(u+\delta) - g(u) &= L(b_1, \ldots, b_{i-1}, u+\delta, b_{i+1}, \ldots, b_p)- L(b_1, \ldots, b_{i-1}, u, b_{i+1}, \ldots, b_p) \\
&  \quad + \lambda_n (\Phi(|u+\delta|) - \Phi(|u|)) \\
& \geq \nabla_i L(b_1, \ldots, b_{i-1}, u, b_{i+1}, \ldots, b_p) \delta + \frac{\gamma_i}{2} \delta^2 \\
& \quad + \lambda_n \Big\{\Phi'(|u|)(|u+\delta| -|u|) + \frac{1}{2} \Phi''(|u^{*}|)(|u+\delta| -|u|)^2 \Big\},
\end{align*}
where $|u^{*}|$ is some number between $|u+\delta|$ and $|u|$.

Assume that $g(u)$ achieves the minimum value at $u_0$. Then $0 \in \partial g(u_0)$. Hence,
\[
\nabla_i L(b_1, \ldots, b_{i-1}, u_0, b_{i+1}, \ldots, b_p) + \lambda_n \Phi'(|u_0|) \sgn(u_0)=0.
\]
Notice that if $u_0=0$, then the above equation holds true for some value in $[-1, 1]$. For notational convenience, we here still write
such a value by $ \sgn(u_0)$.  Additionally,  we have
\[
\Phi'(|u|)(|u+\delta| -|u|)- \Phi'(|u|) \sgn(u) \delta = \Phi'(|u|)(|u+\delta| -|u| -\sgn(u) \delta ) \geq 0
\]
because $\Phi'(|u|)\geq 0$.
We now obtain
\[
g(u_0+\delta) - g(u_0) \geq  \frac{\gamma_i}{2} \delta^2 +   \frac{\lambda_n}{2} \Phi''(|u_0^{*}|)(|u_0+\delta| -|u_0|)^2
\geq \frac{\gamma_i+ \lambda_n \Phi''(0)}{2} \delta^2.
\]
Here we use the fact that $\Phi''(0) \leq \Phi''(|u|) \leq  0$ and $(|u_0+\delta| -|u_0|)^2 \leq \delta^2$.
Let $\rho= \min_{i} \frac{\gamma_i+ \lambda_n \Phi''(0)}{2}$. Then
\[
g(u_0+\delta) - g(u_0) \geq \rho \delta^2.
\]
Applying the above inequality leads us to
\[
F_i^{(t)}(\b_i^{(t)}) - F_i^{(t)}(\b_{i-1}^{(t)}) \geq \rho \|\b_i^{(t)} - \b_{i-1}^{(t)}\|^2,
\]
where $\b_i^{(t)} = (b_1^{(t+1)}, \ldots, b_i^{(t+1)}, b_{i+1}^{(t)}, \ldots, b_p^{(t)})^T$. 
Hence, $F(\b^{(t)}) -F(\b^{(t+1)}) \geq \rho \|\b^{(t)} - \b^{(t+1)}\|^2$. This also implies that
\[
\sum_{t=0}^T \|\b^{(t)} - \b^{(t+1)}\|^2 \leq \frac{1}{\rho} \sum_{t=0}^{T} (F(\b^{(t)}) -F(\b^{(t+1)})) = \frac{1}{\rho}(F(\b^{(0)}) -F(\b^{(T+1)}))< \infty.
\]
Thus, $\sum_{t=0}^{\infty} \|\b^{(t)} - \b^{(t+1)}\|^2 < \infty$.

Consider that 
\[
\nabla_i L(b_1^{(t+1)}, \ldots, b_{i-1}^{(t+1)}, b_i^{(t+1)}, b_{i+1}^{(t)}, \ldots, b_p^{(t)}) + \lambda_n \partial P(b_i^{(t+1)})  =0.
\]
That is,  
\[
\nabla_i L(\b_i^{(t)})  -\nabla_i L(\b^{(t+1)})   + \nabla_i L(\b^{(t+1)})+ \lambda_n \partial P(b_i^{(t+1)})  =0.
\]
This implies that 
\[  w_i^{(t+1)} \in  \nabla_i L(\b^{(t+1)})+ \lambda_n \partial P(b_i^{(t+1)}) =  \partial F_i(\b^{(t+1)}). 
\]  
where $w_i^{(t+1)} = \nabla_i L(\b^{(t+1)})  - \nabla_i L(\b_i^{(t)})  $. Let $\w^{(t+1)}=(w_1^{(t+1)}, \ldots, w_p^{(t+1)})^T $. Then $\w^{(t+1)} \in \partial F(\b^{(t+1)})$. 
Notice that
\begin{align*}
\left | \nabla_i L(\b^{(t+1)})  - \nabla_i L(\b_i^{(t+1)}) \right | & = \Big | \sum_{j =i+1} a_{i j} (b_j^{(t+1)} - b_j^{(t)})   \Big | \\
& \leq  \sum_{j =i+1} | a_{ij} | \big |(b_j^{(t+1)} - b_j^{(t)}) \big| \leq  \sum_{j =i+1}  \big |b_j^{(t+1)} - b_j^{(t)} \big| \\
&  \leq  \||\b^{(t+1)} - \b^{(t)}\|_1 \leq  \sqrt{p} \|\b^{(t+1)} - \b^{(t)} \|,
\end{align*}
where $\X^T \X = \A =[a_{ij}]$ and  $B >0$ is some constant.  
Here the second inequality is due to $|a_{ij}| = |\sum_{l=1}^n x_{ l i} x_{l j}| \leq 1$. Hence,
\[
\|\w^{(t+1)}\| =  \sqrt{\sum_{i=1}^p  \left | \nabla_i L(\b^{(t+1)})  - \nabla_i L(\b_i^{(t+1)}) \right |^2 }  \leq C_1 \|\b^{(t+1)} - \b^{(t)} \|,
\] 
where $C_1 =  {p}$. 
\end{proof}

\subsection{The Proof of Lemma~\ref{lem:critical}}

\begin{proof} As we have mentioned above,  the sequence  $\left\{\b^{(t)}: t \in \NB \right\}$ is bounded. Thus, there exists a convergent subsequence $\left\{\b^{(t_k)}\right\}$ that converges to $\b^*$.  By the continuity of $L$ and $P$, we have
\[
\lim_{k \to \infty} F(\b^{(t_k)}) = F(\b^*). 
\]
In terms of the proof of~\ref{thm:decrease}-(iii), we know that $w_i^{(t_k)} =  \nabla_i L(\b^{(t_k)}) - \nabla_i L(\b_i^{(t_k-1)}) $. Since $\nabla_i  L$ is continuous, it is obtained that 
\[
\lim_{k\to \infty} \w^{(t_k)} =\0 \in F(\b^*). 
\]
This implies that $\b^*$ is a critical point of $F$. By the definition of limit points, we then have 
\[
\lim_{t \to \infty} {\mathrm{dist}}\big(\b^{(t)}, \MM(\b^{(0)})\big)=0. 
\]
\end{proof}

\subsection{The Proof of Theorem~\ref{thm:global}}

\begin{proof}\label{pf:main}
As is known, there exists an increasing sequence $\left\{t_{k}\right\}_{k\in \NB}$ such that $\left\{\b^{(t_k)}\right\}$ converges to $F(\b^*)$.
Suppose there exists an integer $N_0$ such that $F(\b^{(N_0)})=F(\b^*)$. Then it is obvious that for any integer $N>N_0$, $F(\b^{(N)})=F(\b^*)$ holds. Then it is trivial to achieve the convergent sequence. Otherwise, we consider $F(\b^{(t)})>F(\b^*)$,  $\forall t \in \NB$. Because the sequence $F(\b^{(t)})$ is convergent, it is clear that for any $\eta>0$, there exists a positive integer $M_1$ such that $F(\b^{(t)})<F(\b^*)+\eta$ for all $t>M_1$. By using (\ref{eq:set}), we have $\liml_{t \to\infty} \dist(\b^{(t)},  \MM(\b^{(0)}))=0$ which implies that for any $\epsilon>0$ there exists a positive integer $M_2$ such that $\dist(\b^{(t)}, \MM(\b^{(0)}))<\epsilon$ for all $t > M_2$. Let $l = \max \left\{M_1, M_2\right\}$.
Then by the Kurdyka-\L{}ojasiewicz inequality in Definition~\ref{def:klpro}, we have   
\begin{equation*}\label{KL}
\pi^{\prime}\Big(F(\b^{(t)})-F(\b^*)\Big) \dist(0, \partial F(\b^{(t)}))\geq 1 \;  \mbox{ for any  }  \; t > l.
\end{equation*}
By  (\ref{eqn:subinequality}), we have
\begin{equation*}\label{eq:lb}
\pi^{\prime}\Big(F(\b^{(t)})- F(\b^*)\Big) \geq \frac{1}{C_1} \Big\|\b^{(t)} - \b^{(t-1)}\Big\|^{-1}.
\end{equation*}
Let $\Delta_{t,  t+1} \triangleq \pi(F(\b^{(t)}) - F(\b^*))- \pi(F(\b^{(t+1)})- F(\b^*))$. With the property of concave functions, we have

\begin{align*} \label{eq:concave}
\Delta_{t, t+1}&\geq \pi^{\prime}(F(\b^{(t)}) - F(\b^*))(F(\b^{(t)})- F(\b^{(t+1)})) \\
         &\geq \frac{C_0}{2}\pi^{\prime}(F(\b^{(t)}) - F(\b^*))\|\b^{(t+1)}- \b^{(t)} \|^2 \\
         &\geq \frac{C_0 }{2 C_1}\|\b^{(t)}- \b^{(t-1)}\|^{-1} \|\b^{(t+1)}- \b^{(t)} \|^2.
\end{align*}
That is,
\[
 C  \Delta_{t, t+1} \|\b^{(t)} - \b^{(t-1)}\| \geq\|\b^{(t+1)}- \b^{(t)} \|^2,
\]
where  $C = \frac{2 C_1}{C_0}$.
Notice that
\[
  C \Delta_{t, t+1} \|\b^{(t+1)}- \b^{(t)}\| \leq  \Big(\frac{C \Delta_{t, t+1}+ \|\b^{(t+1)} - \b^{(t)}\|}{2}\Big)^2.
\]
We thus have
\[
2 \|\b^{(t+1)}- \b^{(t)} \|\leq C \Delta_{t, t+1}+ \|\b^{(t)}- \b^{(t-1)}\|.
\]
Then
\begin{align*}
\sum^{\infty}_{t=l+1}\|\b^{(t+1)}- \b^{(t)} \|&\leq C \sum^{\infty}_{t=l+1} \Delta_{t, t+1}+ \sum^{\infty}_{t=l+1}\Big(\|\b^{(t)} - \b^{(t-1)}\|-\|\b^{(t+1)} - \b^{(t)} \|\Big) \\
& \leq C  \Delta_{l+1, \infty} + \|\b^{(l+1)} - \b^{(l)} \| \\
&\leq C \pi\Big(F(\b^{(l+1)}) - F(\b^*) \Big) + \|\b^{(l+1)}-\b^{(l)} \|.
\end{align*}
Since $\liml_{l \to \infty} \|\b^{(l+1)} - \b^{(l)} \|=0$ and $\liml_{l \to \infty} F(\b^{(l+1)})= F(\b^*)$, it is clearly seen that
\[
\lim_{l \to \infty}\sum^{\infty}_{t=l+1} \|\b^{(t+1) }- \b^{(t)} \|=0.
\]
Thus we obtain  
\[
 \sum^{\infty}_{t=0} \|\b^{(t+1)} - \b^{(t)} \|<\infty. 
\]
This implies that $\left\{\b^{(t)} \right\}_{t \in \NB}$ is a Cauchy sequence, and hence, it is a convergent sequence that converges to $\b^*$.
\end{proof}

\section{Asymptotic Properties}
\label{sec:math}

We discuss asymptotic properties
of the sparse estimator under the regression setting. Following the setup of \citet{ZouLi:2008} and \citet{ArmaganDunsonLee},
we assume two conditions: (i) $y_i=\x_i^T \b^{*} + \epsilon_i$ where $\epsilon_1, \ldots, \epsilon_n$ are i.i.d.\ errors
with mean 0 and variance $\sigma^2$;
(ii) $\X^T \X/n\rightarrow \C$ where $\C$ is a positive definite matrix. Let ${\cal A}=\{j: b_{j}^{*} \neq 0\}$.
Without loss of generality, we assume that ${\mathcal A}=\{1, 2, \ldots, r\}$ with $r <p$. Thus, partition $\C$ as
\[
\C= \begin{bmatrix}\C_{11} & \C_{12} \\  \C_{21} & \C_{22} \end{bmatrix},
\]
where $\C_{11}$ is $r{\times} r$. Additionally,  let $\b^{*}_{1} =\{b^{*}_{j}: j \in {\mathcal A}\}$
and $\b^{*}_{2} = \{b^{*}_{j}:  j \notin {\cal A}\}$.

We are now interested in the  asymptotic behavior of the sparse estimator based on the penalty function $\Phi(\alpha |b|)$. That is,
\begin{equation} \label{eqn:sqre}
\tilde{\b}_n=\argmin_{\b}  \;  \|\y {-} \X \b\|_2^2 +  \lambda_n \sum_{j=1}^p    {\Phi(\alpha_n |b_j|)}.
\end{equation}
Furthermore, we let $\lambda_n=\frac{\eta_n}{\Phi(\alpha_n)}$ based on Theorem~\ref{thm:lp2}. For this estimator, we have the following  oracle property.

\begin{theorem}  \label{thm:oracle2} Let $\tilde{\b}_{n1}=\{\tilde{b}_{nj}: j \in \AM\}$ and   $\tilde{ {\cal A}}_n=\{j: \tilde{b}_{nj} \neq 0\}$. Suppose $\Phi$ is a Bernstein function on $[0, \infty)$ such that $\Phi(0)=0$ and $\Phi'(0)=1$, and there exists a constant $\gamma \in [0, 1)$
such that $\liml_{\alpha \to \infty} \frac{ \Phi'(\alpha)}{\alpha^{\gamma-1}} = c_0$ where $c_0 \in (0, \infty)$ when  $\gamma \in (0, 1)$ and $c_0 \in [0, \infty)$ when $\gamma=0$.
If $\eta_n/{n}^{\frac{\gamma_1}{2}} \rightarrow c_1 \in (0, \infty)$    and $\alpha_n/{n}^{\frac{\gamma_2}{2}} = c_2 \in (0, \infty)$ where $\gamma_1 \in (0, 1]$ for $\gamma =0$ or $\gamma_1 \in (0, 1)$ for $\gamma > 0$
and $\gamma_2 \in (0, 1]$  such that $\gamma_1{+}\gamma_2>1+\gamma \gamma_2$,
then $\tilde{\b}_n$ satisfies the following properties:
\begin{enumerate}
\item[\emph{(1)}]  Consistency in variable selection: $\liml_{n \rightarrow \infty} P( \tilde{{\cal A}}_n={\cal A})=1$.
\item[\emph{(2)}]  Asymptotic normality:  $\sqrt{n}(\tilde{\b}_{n1} - \b^{*}_{1}) \overset{d}{\longrightarrow} N(\0, \sigma^2 \C_{11}^{-1})$.
\end{enumerate}
\end{theorem}%

Notice that $\Phi'$ is the Laplace transform of some distribution function (say $F$). Based on the Tauberian Theorem~\citep{Widder:1946},
the condition   $\liml_{\alpha \to \infty} \frac{ \Phi'(\alpha)}{\alpha^{\gamma-1}} = c_0$ is equivalent to that
$\liml_{t \to 0+} \frac{F(t)}{t^{1-\gamma}} = \frac{c_0}{\Gamma(2-\gamma)}$.

Obviously, the function $\Phi_{\rho}$ in (\ref{eqn:first}) satisfies the conditions
in Theorem~\ref{thm:oracle2}; that is, we see $\gamma = - \frac{\rho}{1-\rho}$ when $\rho\leq 0$ and $\gamma=0$ when $0< \rho\leq 1$ (see Proposition~\ref{pro:33}).  It follows from the condition $\liml_{\alpha \to \infty} \frac{ \Phi'(\alpha)}{\alpha^{\gamma-1}} = c_0$
that $\liml_{\alpha \to \infty} \frac{ \Phi(\alpha)}{\alpha^{\gamma}} = \frac{c_0}{ \gamma}$ for $\gamma\neq 0$. As a result, we obtain $\liml_{\alpha \to \infty} \frac{\alpha \Phi'(\alpha)}{\Phi(\alpha)} = {\gamma}$.
The condition  $\alpha_n/{n}^{\gamma_2/2} = c_2$ implies that $\alpha_n \to \infty$. Subsequently,
we have $\liml_{n \to \infty} \sum_{j=1}^p \frac{\Phi(\alpha_n |b_j|)} {\Phi(\alpha_n)} = \sum_{j=1}^p |b_j|^{\gamma}$ (see Theorem~\ref{thm:lp2}).
On the other hand, as stated earlier,  $\liml_{\alpha_n \to 0+} \sum_{j=1}^p \frac{\Phi(\alpha_n |b_j|)} {\Phi(\alpha_n)} = \liml_{\alpha_n \to 0+} \sum_{j=1}^p \frac{\Phi(\alpha_n |b_j|)} {\alpha_n} = \|\b\|_1$. Thus, we are also interested in the corresponding  asymptotic behavior of the sparse estimator.
In particular, we have the following theorem.

\begin{theorem}  \label{thm:asumptotic}
Let $\Phi$ be a Bernstein function such that $\Phi(0)=0$ and $\Phi'(0)=1$. Assume
$\liml_{n \to \infty} \alpha_n =0$. If $\liml_{n \to \infty} \frac{\eta_n}{\sqrt{n}} = 2 c_3 \in [0, \infty)$,
then  $\tilde{\b}_{n} \overset{p}{\longrightarrow} \b^{*}$. Furthermore, if $\liml_{n \to \infty} \frac{\eta_n}{\sqrt{n}} =0$,
then $\sqrt{n}(\tilde{\b}_{n} {-} \b^{*})
\overset{d}{\longrightarrow} N(\0, \sigma^2 \C^{-1})$.
\end{theorem}%

In the previous discussion, $p$ is fixed. It  would be also interested in the asymptotic properties when $r$  and $p$ rely on $n$~\citep{ZhaoJMLR:06}. That is, $r \triangleq r_{n}$ and $p \triangleq p_n$ are allowed to grow as $n$ increases.
 Consider that $\tilde{\b}_n$
is the solution of the problem in (\ref{eqn:sqre}). Thus,
\[
0 \in  (\X \tilde{\b}_n {-} \y)^T \x_{\cdot j} + \frac{\eta_n \alpha_n \Phi'(\alpha_n |\tilde{b}_{nj}|)}{\Phi(\alpha_n)}  \partial |\tilde{b}_{nj}|, \quad j=1, \ldots, p.  \]
Under the condition $\alpha_n\to 0$, we have
\[
0 \in \lim_{n \to \infty} \Big\{ (\X \tilde{\b}_n {-} \y)^T \x_{\cdot j} + \frac{\eta_n \alpha_n \Phi'(\alpha_n |\tilde{b}_{nj}|)}{\Phi(\alpha_n)}  \partial |\tilde{b}_{nj}| \Big\}
= \lim_{n \to \infty} \Big\{ (\X \tilde{\b}_n {-} \y)^T \x_{\cdot j} + {\eta_n}  \partial |\tilde{b}_{nj}| \Big\}
\]
for $ j=1, \ldots, p$. Since the minimizer of the conventional lasso exists and unique (denote  $\hat{\b}_{0}$), the above
relationship implies that $\liml_{n \to \infty} \tilde{\b}_n=\liml_{n \to \infty} \hat{\b}_{0}$. Accordingly, we can obtain the same result as
in Theorem~4 of \citet{ZhaoJMLR:06}.

Recently, \citet{ZhangZhang2012} presented a general theory of nonconvex regularization for sparse learning problems.
Their work is built on the following four conditions on the penalty function $P(b)$: (i) $P(0)=0$; (ii)  $P(-b)=P(b)$;
(iii) $P(b)$ is increasing in $b$ on $[0, \infty)$; (iv) $P(b)$ is subadditive w.r.t.\
$b\geq 0$, i.e., $P(s+t) \leq P(s)+ P(t)$ for any $s \geq 0$ and $t \geq 0$.
It is easily seen  that the Bernstein function $\Phi(|b|)$ as a function of $b$ satisfies the first three conditions.
As for the fourth condition, it is also obtained via the fact that
\begin{align*}
\Phi(s+t) & = \int_{0}^{\infty}{ [1 - \exp(-(s+t) u)] \nu(d u) } \\
 & \leq \int_{0}^{\infty}{[1 - \exp(-s u) + 1 - \exp(-t u) ] \nu(d u) }
= \Phi(s) + \Phi(t), \quad \mbox{ for } s, t>0.
\end{align*}
Thus, we can directly apply the theoretical analysis of \citet{ZhangZhang2012}  to the Bernstein nonconvex penalty function.

The Bernstein function $\Phi(|b|)$ studied in this paper also satisfies Assumption~1 made in 
\citet{LohWainwright:2013}  (see  Proposition~\ref{thm:nest}-(a) and (c)).  
This implies that  the theoretical analysis of \citet{LohWainwright:2013}  applies to the Bernstein penalty function.

\subsection{The Proof of Theorems~\ref{thm:oracle2} and \ref{thm:asumptotic}}

The proof is similar to that of Theorem~1 in \citet{ArmaganDunsonLee}.
Let $\tilde{\b}_n =\b^{*} + \frac{\hat{\u} }{\sqrt{n}}$ and
\[
\hat{\u} = \argmin_{\u} \; \bigg\{G_n(\u) \triangleq  \Big\|\y - \X (\b^{*} +\frac{\u}{\sqrt{n}} ) \Big\|_2^2
    + \eta_n \sum_{j=1}^p \frac{\Phi(\alpha_n|b^{*}_j {+} \frac{u_j}{\sqrt{n}}|) }  {\Phi(\alpha_n)} \bigg\}.
\]
Then $\hat{\u}  = \sqrt{n}(\tilde{\b}_n - \b^{*})$.
Consider that
\[
 G_n(\u) - G_n(\0)
 = \u^T(\X^T \X/n) \u - 2 \u^T  \frac{\X^T \epsi}{\sqrt{n}}+
  \eta_n \sum_{j=1}^p
\frac{\Phi(\alpha_n|b^{*}_j {+} \frac{u_j}{\sqrt{n}}|) {-} \Phi(\alpha_n|b^{*}_j|) } {\Phi(\alpha_n)}.
\]
Clearly, $\X^T \X/n \rightarrow \C$ and $\frac{\X^T\epsi}{\sqrt{n}} \overset{d}{\rightarrow} \z  \overset{d}{=} N(\0, \sigma^2 \C)$.
We now discuss the limiting behavior of the third term of the right-hand side.

We partition $\z$ into $\z^T=(\z_1^T, \z_2^T)$ where  $\z_1=\{z_j: j \in {\cal A}\}$ and $\z_2=\{z_j: j \notin {\cal A}\}$.
First, assume $b^{*}_j=0$. The previous results imply
\[
 \eta_n \frac{\Phi(|u_j| \frac{\alpha_n}{\sqrt{n}})} {\Phi(\alpha_n)}
\backsimeq  \frac{n^{\frac{\gamma_1 {+} \gamma_2 {-} 1}{2}}}{n^{\frac{\gamma_2 \rho}{2} }} \frac{\eta_n}{n^{\frac{\gamma_1}{2} }} \frac{\alpha_n}{n^{\frac{\gamma_2}{2} }}  \frac{n^{\frac{\gamma_2 \rho}{2} }}{\alpha_n^{\rho}}  \frac{\alpha_n^{\rho} } {\log (\alpha_n)}  \frac{ \log (\alpha_n) } {\Phi(\alpha_n)}  \frac{\Phi \big(|u_j| \frac{ \alpha_n}{\sqrt{n} } \big) }
 {\frac{ \alpha_n}{\sqrt{n} }} \rightarrow +\infty
\]
whenever $\gamma=0$, due to $\liml_{\alpha \to \infty} \frac{\log(\alpha)}{\Phi(\alpha)} = \liml_{\alpha \to \infty} \frac{1}{\alpha \Phi'(\alpha)} = \frac{1}{c_0}>0$. Here we take $\rho$ as a positive constant such that $\rho\leq \frac{\gamma_1 {+} \gamma_2 {-} 1}{\gamma_2}$. If $\gamma \in (0, 1)$, we also have
\[
 \eta_n \frac{\Phi(|u_j| \frac{\alpha_n}{\sqrt{n}})} {\Phi(\alpha_n)}
\backsimeq \frac{n^{\frac{\gamma_1 +\gamma_2- 1}{2}}}{n^{ \frac{\gamma_2 \gamma}{2}}}  \frac{ \alpha_n^{\gamma} } {\Phi(\alpha_n)}  \frac{\Phi\big(|u_j| \frac{\alpha_n}{\sqrt{n}} \big) }
 { \frac{\alpha_n}{\sqrt{n}}} \rightarrow +\infty,
\]
because $\liml_{\alpha \to \infty} \frac{\alpha^{\gamma}}{\Phi(\alpha)} = \liml_{\alpha \to \infty} \frac{\gamma \alpha^{\gamma-1}}{ \Phi'(\alpha)} = \frac{\gamma}{c_0}>0$.

Next, we assume that $b^{*}_j\neq 0$. Subsequently, for sufficiently large $n$,
\begin{align} \label{eqn:99}
& \eta_n \frac{\Phi(\alpha_n|b^{*}_j {+} \frac{u_j}{\sqrt{n}}|) - \Phi(\alpha_n|b^{*}_j|) } {\Phi(\alpha_n)} \nonumber  \\
 & =   \eta_n \frac{\Phi(\alpha_n (b^{*}_j {+} \frac{u_j}{\sqrt{n}})
 \sgn(b^{*}_j) ) {-} \Phi(\alpha_n b^{*}_j \sgn(b^{*}_j)) } {\Phi(\alpha_n)}  \nonumber \\
& =  \frac{u_j}{b^{*}_j {+} \theta \frac{u_j}{\sqrt{n}}} \frac{\eta_n }{\sqrt{n}}  \frac{ \Phi'\Big(\alpha_n (b^{*}_j {+} \theta \frac{u_j}{\sqrt{n}}) \sgn(b^{*}_j) \Big)  \alpha_n (b^{*}_j {+} \theta \frac{u_j}{\sqrt{n}}) \sgn(b^{*}_j)} {\Phi(\alpha_n)}  \quad \{\mbox{for some }  \theta \in (0, 1) \}  \\
& \to 0. \nonumber
\end{align}
Here we use the fact that $\liml_{z \to \infty} \frac{z \Phi'(z)}{\Phi(z)} = \gamma \in [0, 1)$.

By Slutsky's theorem, we have
\[
 G_n(\u) - G_n(\0)  \overset{d}{\rightarrow} \left\{\begin{array}{ll} \u_1^T \C_{11} \u_1 - 2 \u_1^T  \z_1 & \mbox{ if } u_j=0 \; \forall j \notin {\cal A},
 \\ \infty & \mbox{ otherwise}. \end{array} \right.
\]
This implies that $G_n(\u) - G_n(\0)$ converges in distribution to a convex function, whose unique minimum is
$(\C_{11}^{-1} \z_1, \0)^T$. It then follows from  epiconvergence \citep{KnightFu:2000} that
\begin{equation} \label{eqn:11}
\hat{\u}_1 \overset{d}{\rightarrow} \C_{11}^{-1} \z_1 \; \mbox{ and } \;  \hat{\u}_2 \overset{d}{\rightarrow} \0.
\end{equation}
This proves asymptotic normality due to $\z_1 \overset{d}{=} N(\0, \sigma^2 \C_{11})$.

Recall that $\tilde{b}_{n j} \overset{p}{\rightarrow} b^{*}_j$ for any $j \in \AM$, which implies that $\Pr(j \in \AM_n) \to 1$.
Thus, for consistency in Part (1), it suffices to obtain  $\Pr(l \in \AM_n) \to 0$ for any $l \notin \AM$.
For such an event ``$l\in \AM_n$," it follows from the KKT optimality conditions
that $2 \x_{l}^T(\y - \X \tilde{\b}_{n})=\frac{\eta_n \alpha_n \Phi'(\alpha_n |\tilde{b}_{nj}|)}{\Phi(\alpha_n)}$.
Notice that
\[
\frac{2 \x_l^T (\y -\X \tilde{\b}_{n})}{\sqrt{n}} = 2 \frac{\x_l^T \X \sqrt{n}(\b^{*}-\tilde{\b}_n)}{n} + \frac{2 \x_l^T \epsi}{\sqrt{n}},
\]
and $\liml_{n\to \infty} \frac{\eta_n \alpha_n \Phi'(\alpha_n |\tilde{b}_{nj}|)}{\sqrt{n} \Phi(\alpha_n)}
= \liml_{n\to \infty} \frac{\eta_n \alpha_n \Phi'( \sqrt{n} |\tilde{b}_{nj}| \alpha_n/\sqrt{n})}{\sqrt{n} \Phi(\alpha_n)}
\backsimeq \liml_{n\to \infty} \frac{n^{\gamma_1 {+} \gamma_2 {-} \frac{1}{2}} }{\gamma_2 \log(n)} \frac{\log(\alpha_n)} {\Phi(\alpha_n)} \to \infty$ for $\gamma=0$ or $\liml_{n\to \infty} \frac{\eta_n \alpha_n \Phi'(\alpha_n |\tilde{b}_{nj}|)}{\sqrt{n} \Phi(\alpha_n)}
= \liml_{n\to \infty} \frac{\eta_n \alpha_n \Phi'( \sqrt{n} |\tilde{b}_{nj}| \alpha_n/\sqrt{n})}{\sqrt{n} \Phi(\alpha_n)}
\backsimeq \liml_{n\to \infty} \frac{n^{\gamma_1 {+} \gamma_2 {-} \frac{1}{2}} }{n^{\frac{\gamma \gamma_2}{2} } } \frac{ \alpha_n^{\gamma} } {\Phi(\alpha_n)} \to \infty$ for $\gamma>0$
due to $\sqrt{n} |\tilde{b}_{nj}| \overset{p}{\rightarrow}  0$ by (\ref{eqn:11}) and Slutsky's theorem. Accordingly, we have
\[
\Pr(l \in \AM_n) \leq \Pr\Big[ 2 \x_{l}^T(\y - \X \tilde{\b}_{n})=\frac{\eta_n \alpha_n \Phi'(\alpha_n |\tilde{b}_{nj}|)}{\Phi(\alpha_n)}\Big]
\to 0.
\]

As for the proof of Theorem~\ref{thm:asumptotic}, we consider the case that $\liml_{n\to \infty} \alpha_n =0$. In this case, we have
\[
\lim_{n \to \infty} \frac{\Phi(\alpha_n/\sqrt{n})}{\alpha_n/\sqrt{n}} = 1 \; \mbox{ and } \;
\lim_{n \to \infty} \frac{\Phi(\alpha_n)}{\alpha_n} = 1.
\]
Assume that $\liml_{n\to \infty} \eta_n/\sqrt{n}= 2 c_3 \in [0, \infty]$. Then
\[
 \eta_n \frac{\Phi(|u_j| \frac{\alpha_n}{\sqrt{n}})} {\Phi(\alpha_n)}= |u_j|  \frac{\eta_n}{\sqrt{n}}
 \frac{\alpha_n  \Phi(|u_j| \frac{\alpha_n}{\sqrt{n}})} {\Phi(\alpha_n) |u_j| {\alpha_n}/{\sqrt{n}}} \to 2 c_3 |u_j|
\]
when $u_j\neq 0$. If $b^{*}_j\neq 0$, then
\begin{align*}
& \eta_n \frac{\Phi(\alpha_n|b^{*}_j {+} \frac{u_j}{\sqrt{n}}|) - \Phi(\alpha_n|b^{*}_j|) } {\Phi(\alpha_n)} \nonumber  \\
 & =   \eta_n \frac{\Phi(\alpha_n (b^{*}_j {+} \frac{u_j}{\sqrt{n}})
 \sgn(b^{*}_j) ) {-} \Phi(\alpha_n b^{*}_j \sgn(b^{*}_j)) } {\Phi(\alpha_n)}  \nonumber \\
& =  \frac{u_j}{b^{*}_j {+} \theta \frac{u_j}{\sqrt{n}}} \frac{\eta_n }{\sqrt{n}}  \frac{ \Phi'\Big(\alpha_n (b^{*}_j {+} \theta \frac{u_j}{\sqrt{n}}) \sgn(b^{*}_j) \Big)  \alpha_n (b^{*}_j {+} \theta \frac{u_j}{\sqrt{n}}) \sgn(b^{*}_j)} {\Phi(\alpha_n)}  \quad \{\mbox{for some }  \theta \in (0, 1) \}  \\
& \to 2 c_3 u_j \sgn(b^{*}_j).
\end{align*}
We now first consider the case that $c_3 = 0$. In this case,
we have
\[
 G_n(\u) - G_n(\0)  \overset{d}{\longrightarrow} \u^T \C \u - 2 \u^T  \z,
\]
which is convex w.r.t.\ $\u$. Then the minimizer of $\u^T \C \u {-} 2 \u^T  \z$ is $\u^{*}$ if and only if  $\C \u^{*} - \z=\0$. Since
$\hat{\u} \overset{d}{\rightarrow} \u^{*}$ (by epiconvergence), we obtain $ \sqrt{n}(\tilde{\b}_n - \b^{*})=\hat{\u}  \overset{d}{\rightarrow}
N(\0, \sigma^2 \C^{-1})$.

We then consider the case that $c_3 \in (0, \infty)$. Right now we have
\[
 G_n(\u) - G_n(\0)  \overset{d}{\longrightarrow} \u^T \C \u - 2 \u^T  \z+ 2c_3 \sum_{j \in \AM} u_j \sgn(b^{*}_j) +
2 c_3 \sum_{j \notin \AM}  |u_j|  \triangleq H_2(\u).
\]
$H_2(\u)$ is convex in $\u$. Let the minimizer of $H_2(\u)$ be $\u^{*}$. Then
\[
\C \u^{*} - \z + c_3 \s =0
\]
where $\s^T = (\sgn(\b^{*}_1)^T, \v^T)$ and $\v \in \RB^{p_2}$ with $\max_{j} |v_j| \leq 1$. Thus,
we have $\u^{*} \overset{d}{\rightarrow} N({\bf t}, \sigma^2 \Tha)$ where ${\bf t}=(t_1, \ldots, t_p)^T=-c_3 \C^{-1} \s$
and $\Tha=[\theta_{ij}]= \C^{-1}$.
For any $\epsilon>0$, when $n$ is significantly large and using Chebyshev's inequality,  we have that
\begin{align*}
\Pr\Big[|u_j^{*}|/\sqrt{n} \geq \epsilon \Big]
&= \Pr\Big[ |u_j^{*}| \geq \sqrt{n} \epsilon \Big] \\
& \leq
\Pr\Big[|u_j^{*} - t_j| \geq \sqrt{n} \epsilon - |t_j| \Big]
 \leq \frac{\sigma^2 \theta_{jj}}{(\sqrt{n} \epsilon - |t_j| )^2} \to 0
\end{align*}
for $j=1, \ldots, p$. Consequently,  $|u_j^{*}|/\sqrt{n} \overset{p}{\rightarrow}0$;
that is, $\tilde{\b}_n  \overset{p}{\rightarrow} \b^{*}$.

\bibliography{ncvs2}

\end{document}